\documentclass[manuscript]{biometrika}
\usepackage{amsmath,amssymb}
\usepackage{times}
\usepackage{bm}
\usepackage{natbib}
\usepackage{color}
\usepackage{graphicx}
\usepackage[plain,noend]{algorithm2e}
\makeatletter
\renewcommand{\algocf@captiontext}[2]{#1\algocf@typo. \AlCapFnt{}#2} 
\def\@algocf@capt@plain{top}
\renewcommand{\algocf@makecaption}[2]{%
  \addtolength{\hsize}{\algomargin}%
  \sbox\@tempboxa{\algocf@captiontext{#1}{#2}}%
  \ifdim\wd\@tempboxa >\hsize
    \hskip .5\algomargin%
    \parbox[t]{\hsize}{\algocf@captiontext{#1}{#2}}
  \else%
    \global\@minipagefalse%
    \hbox to\hsize{\box\@tempboxa}
  \fi%
  \addtolength{\hsize}{-\algomargin}%
}
\makeatother

\def\BIC{\textsc{bic}}
\def\SCAD{\textsc{scad}}

\def\LASSO{\text{lasso}}
\def\T{{ \mathrm{\scriptscriptstyle T} }}

\def\pr{\hbox{pr}}
\newcommand{\B}{\mathcal{B}}

\newcommand{\Cov}{{\mbox{cov}}}
\newcommand{\Var}{{\mbox{var}}}

\def\trans{^{\rm T}}

\begin{document}

\jname{Biometrika}
\jyear{2015}
\jvol{?}
\jnum{?}
\accessdate{Advance Access publication on ? ? 2015}
\copyrightinfo{\Copyright\ 2012 Biometrika Trust\goodbreak {\em Printed in Great Britain}}

\received{February 2015}
\revised{October 2015}

\markboth{X. Gao \and R. J. Carroll}{X. Gao \and R. J. Carroll}

\title{Data Integration with High Dimensionality }

\author{XIN GAO}
\affil{Department of Mathematics and Statistics,
York University, Toronto, Ontario, Canada M3J 1P3 \email{xingao@mathstat.yorku.ca} }

\author{\and RAYMOND J. CARROLL}
\affil{Department of Statistics,
 Texas A\&M University, College Station, Texas, 77843-3143 U.S.A.,  \\
 \email{carroll@stat.tamu.edu}}

\maketitle

\begin{abstract}
We consider a problem of data integration. Consider determining which genes affect a disease. The genes, which we call predictor objects, can be measured in different experiments on the same individual. We address the question of finding which genes are predictors of disease by any of the experiments. Our formulation is more general. In a given data set, there are a fixed number of responses for each individual, which may include a mix of discrete, binary and continuous variables. There is also a class of predictor objects, which may differ within a subject depending on how the predictor object is measured, i.e., depend on the experiment. The goal is to select which predictor objects affect any of the responses, where the number of such informative predictor objects or features tends to infinity as sample size increases. There are marginal likelihoods for each way the predictor object is measured, i.e., for each experiment.  We specify a pseudolikelihood combining the marginal likelihoods, and propose  a pseudolikelihood information criterion. Under regularity conditions, we establish selection consistency for the pseudolikelihood information criterion with unbounded true model size, which includes a Bayesian information criterion with appropriate penalty term as a special case. Simulations indicate that data integration improves upon, sometimes dramatically, using only one of the data sources.
\end{abstract}


\section{INTRODUCTION}\label{sec1}

Consider the following simple but common problems, both being examples of data integration.
\begin{example} 
\label{Example1}
We have a set of individuals whose disease status is observed. We also measure different facets of individual genes, e.g., mRNA expression, protein expression, RNAseq expression, etc. The question is: which genes affect the disease in any of the different ways the genes are measured? In this example, the gene is really a predictor object, which can be assessed in a number of ways through different measurement processes or experiments. \end{example}

\begin{example} 
\label{Example2} Suppose that the individual is assessed through various responses, measurement mechanisms or experiments, while the predictor or predictor object is the same across these experiments, and we want to examine which predictor affects any of the responses. 
\end{example} 

We consider a formulation that includes both of these cases as well as combinations of them. We recognize that the marginal probability densities among experiments will be different and the measurements from different experiments can be correlated, as they would be in both of the examples described above. Our goal is to show how to combine the various marginal likelihoods and perform inference based on a pseudolikelihood and an information criterion that we develop, doing so in such a way as to allow the number of informative predictor objects or features to tend to infinity as the sample size increases.

One way to approach Example \ref{Example1} is to pool the different means of measuring the gene object and apply a version of the group \text{lasso}, the group being the gene. The group penalty was first formulated in a 1999 Australian National University PhD thesis by Bakin and later proposed to solve the group selection problems by Yuan \& Lin (2006). The group penalty penalizes the $L_2$ norm of the grouped parameter vector, and thus it is able to select predictors based on its overall strength across experiments. Alternatively, other penalty functions such as the smoothly clipped absolute deviation penalty (Fan \& Li, 2001) and the minimax concave penalty (Zhang, 2010) can also be applied in the group penalization scheme. 

The group penalization of pooled parameters of the same covariates is not appropriate for Example \ref{Example2}, nor is it applicable to combinations of Examples \ref{Example1}--\ref{Example2}. A joint model is needed for the multiple responses across the experiments. If a joint model is difficult to specify, pooling all marginal likelihoods together is appropriate for all the examples discussed above. However, to the best of our knowledge, the asymptotic properties of group penalized estimation using pseudolikelihood have not been studied in the literature. 

If there is one response and one set of covariates,  the extended Bayesian information criterion with appropriate penalty term has been shown to be selection consistent, where the total number of predictors tends to infinity and the number of true predictors is bounded by a constant (Chen \& Chen, 2008). Foster \& George (1994) proposed a risk inflation criterion for multiple regression. To handle settings where the number of true predictors is unbounded, Zhang \& Shen (2010) proposed a  corrected risk inflation criterion, and Kim et al. (2012)  proposed a generalized information criterion with modified penalty terms.  The consistency of both criteria has been established only for the linear regression model. It remains an open question of how to design the penalty term for an information criterion to deal with a varying true model size in likelihood models. We aim to find the appropriate penalty term for the Bayesian Information Criterion (\BIC) under likelihood settings when the true model size is unbounded. Further, we extend the results to a pseudolikelihood information criterion, thus including both Examples \ref{Example1}--\ref{Example2} and combinations of them.  

Pseudolikelihood ratio-type statistics do not follow a chi-square distribution but instead asymptotically follow a weighted chi-square distribution. The asymptotic distribution cannot directly provide an upper bound for the tail probability at a given sample size. Sharp deviation bounds have been computed by Spokoiny \& Zhilova (2013) for quadratic forms based on their exact distributions instead of their asymptotic distributions under an  exponential moment condition. We will use  large deviation theory on quadratic forms to obtain the upper bounds of the tail probabilities at any given sample size.  Our work establishes the consistency of  a pseudolikelihood information criterion for divergent true model size.

\section{PSEUDOLIKELIHOOD FORMULATION OF DATA INTEGRATION}\label{sec2}

Using the terminology of Section \ref{sec1}, consider a setting with predictors $M_1$, $\ldots,$ $M_P$ contributing to $k=1,\ldots,K $ different experiments. The overall objective is to integrate the data collected from all the experiments to make inference about the effects of the predictor groups on the process. Given $n$ independent experimental subjects, the data from the $k$th experiment is denoted as $Y_k=(Y_{k1},\ldots, Y_{kn}).$  The parameter vector $\theta_k$ consists of $(\theta_{k1},\ldots,\theta_{kP}),$ where $\theta_{kp}$ denotes the effect of predictor object $M_p$ in experiment $k.$

Data from the $k$th experiment has the likelihood function $L_k(\theta_k;Y_k)=\prod_{i=1}^{n} f_k(Y_{ki};\theta_k),$ where $f_k$ denotes the density function. The densities from different experiments can be of different types including binary, discrete or continuous ones.  Denote all the parameters together as $\theta=(\theta_1,\dots,\theta_k).$ The parameters associated with predictor $M_p$ across different experiments can be grouped together as $\theta^{(p)}=(\theta_{1p},\ldots, \theta_{Kp}),$ which summarizes all the effects of predictor group $p$ across various experiments. The measurements $Y_{1i},\ldots, Y_{Ki}$ may be taken from the same subject or correlated subjects. Therefore, without loss of generality, we assume that $Y_{(i)}=(Y_{1i},\ldots, Y_{Ki})$ have a correlation structure. Of course, the joint distribution of $Y_{(i)}$ may be hard to specify, especially when all the marginal densities are of different types. Table 1 illustrates the set-up for data integration when all the $K$ measurements of $Y_{(i),}$ $i=1,\ldots,n,$ are observed. If some $Y_{(i)}$s are incomplete, an indicator $Z_{ki}$ can be introduced. If $Y_{ki}$ is observed, $Z_{ki}=1,$ otherwise $Z_{ki}=0$. In order to integrate all the experiments, we propose to describe the overall data using a working-independence pseudo-loglikelihood
\begin{align*}
\ell_I(\theta)=\sum_{k=1}^K w_k \ell_k(\theta_k;Y_k)=\sum_{k=1}^K w_k \sum_{i=1}^{n} Z_{ki}\log \{f_k(Y_{ki};\theta_k)\},
\end{align*}
with positive weights $w_k$ $(k=1,\ldots,K).$

\begin{table}
\label{tab1}
\begin{center}
\caption{\baselineskip=12pt Multiple experiments and their parameters. The predictor objects are $M_1,\ldots,M_P$, and the parameter for predictor $M_p$ in experiment $k$ is $\theta_{kp}$.}
\begin{tabular}{ccccc}
& Experiment 1 &Experiment 2 &$\cdots$ &Experiment K\\
parameters & $\theta_1=(\theta_{11},\ldots,\theta_{1P})\trans$ &$\theta_2=(\theta_{21},\ldots,\theta_{2P})\trans$ &$\cdots$&$\theta_K=(\theta_{K1},\ldots,\theta_{KP})\trans$\\
densities & $f_1(Y_1;\theta_1)$ &$f_2(Y_2;\theta_2)$ &$\cdots$ &$f_K(Y_K;\theta_K)$ \\
subject 1 & $Y_{11}$ & $Y_{21}$ & $\cdots$ &$Y_{K1}$\\
subject 2 & $Y_{12}$ & $Y_{22}$ & $\cdots$ &$Y_{K2}$\\
$\vdots$ & $\vdots$ & $\vdots$ &$\ddots$ &$\vdots$\\
subject n & $Y_{1n}$ & $Y_{2n}$ &$\cdots$ & $Y_{Kn}$\\
& $Y_1=(Y_{11},\ldots,Y_{1n})\trans$ & $Y_2=(Y_{21},\ldots,Y_{2n})\trans$ &$\cdots$ &$Y_K=(Y_{K1},\ldots,Y_{Kn})\trans$\\
\end{tabular}
\end{center}
\end{table}

This formulation is similar to composite likelihood (Lindsay, 1988; Cox \& Reid, 2004; Varin, 2008), which combines marginal densities from a multivariate distribution. Nevertheless, it is a new extension in the sense that the marginal densities can come  from different types of distributions. Pseudolikelihood estimation and inference with regard to $\theta$ follows standard theory (White, 1982; Lindsay, 1988; Cox \& Reid, 2004; Varin, 2008; Ribatet et al., 2012). The maximum pseudolikelihood estimate is denoted by $\widehat{\theta}_I=\text{argmax}_{\theta} \ell_I(\theta),$ and it is consistent under regularity conditions. The asymptotic covariance matrix of the maximum pseudolikelihood estimator is given by the inverse of the Godambe information matrix $
G(\theta)=H(\theta)\trans V^{-1}(\theta)H(\theta),
$
where $H(\theta)=E\{-  \partial^2 \ell_I(\theta)
/\partial \theta \partial \theta\trans\}$ and
$V(\theta)=\text{cov}\{ \partial \ell_I(\theta)/\partial \theta\}$ (Godambe, 1960). For inference about $\theta,$ pseudolikelihood ratio statistics and Wald type statistics can be formed.  In the data integration set-up, uniform weights can be assigned to each likelihood. If some experiments have better quality than the others, one might assign them higher weights. In theory, optimal weights can be constructed by projecting the full likelihood score function to the linear space of the composite score functions. However, such optimal weights are challenging to obtain (Lindsay et al., 2011). Some practical strategies for choosing weights based on data structure are given in Varin \& Vidoni (2006) and Joe \& Lee (2009).

\section{FEATURE SELECTION}\label{sec3}
Given multiple experiments with high dimensional predictor objects/features, one can perform penalized estimation to select nonzero features. If one feature $\theta^{(p)}$ is zero, all the corresponding parameters $\theta_{kp}$ $(k=1,\ldots, K)$ are zero simultaneously. Otherwise, at least one of the parameters $\theta_{kp}$ is nonzero. Selecting significant features is equivalent to selecting a group of parameters. We define the overall strength of the predictor $M_p$  as a summarization of all the effect sizes in $\theta^{(p)},$ represented by the $L_2$ norm of $\theta^{(p)}.$ Therefore, we consider the overall objective function
\begin{eqnarray}
Q(\theta)=\ell_I(\theta)-n\sum_{p}\Omega_{\lambda_n}(||\theta^{(p)}||), \label{eqR01}
\end{eqnarray}
with $\Omega_{\lambda_n}$ being the penalty function, and $||\theta^{(p)}||=(\sum_{k=1}^{K}\theta_{kp}^2)^{1/2}$ denoting the $L_2$ norm.

As mentioned previously, standard group selection of variables applies to Example \ref{Example1} but not to Example \ref{Example2}. Yuan \& Lin (2006) considered the problem of group selection and proposed the group \LASSO\ along with corresponding algorithms. Meier et al. (2008) investigated the group \LASSO\ for logistic regression. They showed that the group \LASSO\ yields sparse estimates which are globally consistent in terms of estimation error.  Nardi \& Rinaldo (2008), Bach (2008) and Zhao et al. (2009) proved selection consistency of the group \LASSO\ under regularity conditions.  While the group \LASSO\ possesses excellent properties in terms of prediction and estimation errors, its variable selection consistency depends on the restrictive assumption of a so-called irrepresentability condition, which requires low correlations between significant and insignificant predictors. This condition is difficult to satisfy when $p\gg n$ (Huang et al., 2012).  The group \LASSO\ tends to over-shrink large parameters, because the rate of penalization does not change with the size of the parameters, which then leads to biased estimates of large parameters (Fan \& Li, 2001). Besides the \LASSO\ penalty (Tibshirani, 1996),  many other types of penalty functions have been proposed, including the smoothly clipped absolute deviation penalty (Fan \& Li, 2001) and the minimax concave penalty  (Zhang, 2010). These two penalties can achieve both selection consistency and asymptotic unbiasedness. This oracle property was extended to the group smoothly clipped absolute deviation penalty and the group minimax concave penalty in Wang et al. (2008),
Huang et al. (2012), and Guo et al. (2015). However, none of the existing literature deals with grouped penalization of a pseudolikelihood, which is required for Example \ref{Example2}.  In this paper, we focus on the grouped smoothly clipped absolute deviation penalty for pseudolikelihood and establish its oracle property in high dimensional models.

The smoothly clipped absolute deviation penalty
function satisfies $\Omega_{\lambda}(0)=0,$ and its first-order
derivative is
\begin{align*}
\Omega_{\lambda}'(\theta)=\lambda\{I(\theta\leq \lambda)+
\frac{(a\lambda-\theta)_{+}}{(a-1)\lambda}I(\theta>
\lambda)\},\,\,\theta\geq 0,
\end{align*}
where $a$ is a constant usually set to $3.7$ (Fan \& Li,
2001), and $(t)_{+}=t I(t>0)$ is the hinge loss function.

Denote the total number of features by $p_n,$ and let $p_n \rightarrow \infty$ as $n \rightarrow \infty.$ We assume that $||\theta^{(p)}||>0,$ for $p=1,\ldots,q_n,$ and  $||\theta^{(p)}||=0,$ for $p=q_n+1,\ldots,p_n.$  Define the collection of parameters corresponding to nonzero features and zero features as $\theta_a=(\theta^{(1)},\ldots,\theta^{(q_n)}),$ and $\theta_b=(\theta^{(q_n+1)},\ldots,\theta^{(p_n)}),$ respectively. 

We assume the regularity Conditions 1, 2, and 3 in Appendix 1, which are analogous to those used in Xu \& Reid (2011) and Kwon \& Kim (2012). Below are additional assumptions. 
\begin{assumption}
\label{kappa} Let $\theta^*$ denote the true value of $\theta,$ which is an interior point of the parameter space $\Theta.$ Assume that there exists an integer $\kappa \geq 1$ such that, for constants $(M_1, M_2, M_3)$, $
E_{\theta^*}\{\partial \log f_k(Y_{ki};\theta_k)/\partial \theta_{kj} \}^{2\kappa}\leq M_1,
$ $
E_{\theta^*}\{\partial^2 \log f_k(Y_{ki};\theta_k)/\partial \theta_{kj} \partial \theta_{kl}  \}^{2\kappa}\leq M_2,
$   $
E_{\theta^*}[\{\partial \log f_k(Y_{ki};\theta_k)/\partial \theta_{kj}\}\{ \partial \log \{f_k(Y_{ki};\theta_k)\}/\partial \theta_{kl}  \}]^{2\kappa}\leq M_3,$   $(j,l =1,\ldots,p_n; k=1,\ldots,K).$ 
\end{assumption}

Assumption 1 specifies the boundedness of moments of order $2\kappa$ for the loglikelihood derivatives, which is used to bound certain tail probabilities. For example, if the density is binomial, and $\text{logit}\{{\rm pr}(Y_{ki}=1|\theta_k)\}=X_{ki}^\T\theta_k,$ where $X_{ki}=(X_{ki1},\ldots,X_{kip_n})^\T$ are regression covariates,
then 
\begin{equation}
\label{binom}
\partial \log f_k(Y_{ki};\theta_k)/\partial \theta_{k}=[Y_{ki}-\text{exp}(X_{ki}^\T\theta_k)/\{1+\text{exp}(X_{ki}^\T\theta)\}]X_{ki}.
\end{equation}
If the regression covariates are uniformly bounded in absolute value by a constant $b$, we have $$
\max_{1 \leq j \leq p_n} E_{\theta^*}\{\partial \log f_k(Y_{ki};\theta_k)/\partial \theta_{kj}\}^{2\kappa}\leq \max_{1 \leq j \leq p_n} X_{kij}^{2\kappa}
\leq b^{2\kappa}.
$$
Similarly it can be verified that in generalized linear models, other densities from exponential families satisfy this assumption as long as the absolute values of the regression covariates are uniformly bounded.

\begin{assumption}
\label{constants} There exist two constants $c_1$ and $c_2,$ satisfying $0< 5c_1 <c_2 <1,$ $q_n=o(n^{c_1}),$ and $\min_{1\leq j \leq q_n} n^{(1-c_2)/2}||\theta^{*(j)}||\geq M_5.$
\end{assumption}

Assumption 2 specifies the rate that $q_n$ grows with respect to $n$, and the rate at which the size of the nonzero predictors can approach zero. This means that the proportion of true predictors has to be less than one fifth of the sample size, whereas the potential number of predictors $p_n$ can be greater than $n.$ 

Define the oracle estimate $\widehat{\theta}$ as any local maximizer of the pseudo-loglikelihood $\ell_I(\theta)$ subject to
$||\widehat{\theta}^{(j)}||=0,$ for $j> q_n$  and $||\widehat{\theta}-\theta^*||=O_p\{(q_n/n)^{1/2}\}.$  Under regularity Conditions 1--3 and Assumptions 1--2,
it can be established that such an oracle estimate exists (Theorem 1 in Fan \& Peng, 2004).

Because the penalty function is singular at the origin, we consider the subderivatives of the objective function. The subdifferential of a function is a set-valued mapping and it is a generalized version of derivatives for non-differentiable functions. Taking the subderivative of $Q(\theta)$ in (\ref{eqR01}) with respect to the $j$th grouped parameters $\theta^{(j)}$, we have that
\begin{align}
\label{kkt1}
  \frac{\partial Q(\theta)}{\partial \theta^{(j)}}=\begin{cases}
   \partial \ell_I(\theta)/\partial  \theta^{(j)}-n \lambda_n \text{Sign}(\theta_j), & \text{ $||\theta^{(j)}||\leq \lambda_n$};\\
   \partial \ell_I(\theta)/\partial \theta^{(j)}-n \text{Sign}(\theta_j)\{a\lambda_n-||\theta^{(j)}||\}/(a-1), & \text{ $\lambda_n<||\theta^{(j)}||< a\lambda_n$};\\
    \partial \ell_I(\theta)/\partial  \theta^{(j)}, & \text{   $a\lambda_n\leq ||\theta^{(j)}||$},
  \end{cases}
\end{align}
with $\text{Sign}(\cdot)$ denoting a set-valued map for a real vector. Let $0$ denote the vector of zeros. When $u\neq 0,$ $\text{Sign}(u)$ returns $u/|| u||,$ and for $u=0,$ $\text{Sign}(u)$ returns a set of all possible vectors $\omega$ such that $||\omega||\leq 1.$

\begin{theorem}
\label{thm1}
Let $S(\lambda_n)$ denote the set of solutions to the subdifferential equation $\partial Q(\theta)/\partial \theta$$=0.$ Under regularity Conditions 1--3 and Assumptions 1--2, ${\rm pr}\{\widehat{\theta}\in S(\lambda_n)\}\rightarrow 1,$ provided that $\lambda_n=o\{n^{-(1-c_2+c_1/2)}\}$  and $p_n/(n^{1/2}\lambda_n)^{2\kappa}\rightarrow 0$ as $n\rightarrow \infty.$
\end{theorem}

We emphasize that $p_n$ may be much larger than $n$, provided $\kappa$ defined in Assumption \ref{kappa} is sufficiently large. If the first, second and third derivatives of the pseudo-loglikelihood have exponentially decaying tails, the theorem holds when $p_n=O\{\text{exp}(n^{c_3})\}$ for some constant $c_3>0$ (Kwon \& Kim, 2012).

\begin{theorem}
\label{thm2}
With probability tending to 1, as $n\rightarrow \infty,$ the root $(n/q_n)$ consistent oracle estimate $\widehat{\theta}=(\widehat{\theta}_a,\widehat{\theta}_b)$ in Theorem \ref{thm1} satisfies
$$
n^{1/2} A_n  \{V^{(1)}(\theta^*)\}^{-1/2} H^{(1)}(\theta^*) (\widehat{\theta}_a-\theta_{a}^*) \rightarrow N(0,G),\,
$$
where $V^{(1)}(\theta^*)$ and $H^{(1)}(\theta^*)$ are the submatrices of $V(\theta^*)$ and $H(\theta^*)$ with respect to $\theta_a,$  $\{V^{(1)}(\theta^*)\}^{1/2}$ is the symmetric square root of $V^{(1)}(\theta^*),$ $A_n$ is a $m\times q_n^*$ matrix such that $A_n A_n\trans \rightarrow G,$ where $G$ is a $m\times m$ nonnegative definite symmetric matrix, and $q_n^*=K\times q_n$. 
\end{theorem}

Group penalization has been studied only in true likelihood settings in the literature. We establish the oracle property of group penalization in the pseudolikelihood setting.  Our results show that group penalization using the smoothly clipped absolute deviation penalty is asymptotically model selection consistent even when all the marginal likelihoods are correlated.   For group \text{lasso} to be model selection consistent, the irrepresentability condition (Zhao and Yu, 2006, Meinshausen and B\"{u}hlmann, 2006, Zou, 2006, Bach 2008) is required. Thus the capacity for group lasso to be selection consistent is constrained regardless of the strength of the model signals (Fan and Lv, 2010). In contrast, the group smoothly clipped absolute deviation penalty does not require such a stringent condition. Provided the coefficient sizes of the non-zero parameters are sufficiently far from zero at the rate specified in Assumption 2, the oracle property of the group smoothly clipped absolute deviation penalty can be shown.

\section{PSEUDOLIKELIHOOD INFORMATION CRITERION}\label{sec4}

Although different data sources have various densities and parameters, in our context they share the same set of predictors. Aggregating different information criteria can boost the power to select the correct set of predictors. Given all the competing models, the notion of
consistent model selection is about identifying the smallest
correct model with probability tending to one. Let $s$ be a subset of $(1,\ldots,p_n).$ The model with $\theta^{(p)}=0$ for all $p\notin s,$ is called model $s$. The
sets of under-fitting models and over-fitting models are denoted
as $S_{-}$ and $S_{+},$ respectively. Assume that the largest model size in model space $s\in \mathcal{S}$ is $s_n,$ where $q_n \leq s_n \leq p_n.$

We propose to aggregate the information in a linear manner. Our proposed pseudolikelihood information criterion
is
\begin{align}
\label{eq:clbic} \mbox{pseu-BIC}(s)=-2 \ell_I
(\widehat{\theta}_I;Y)+d_s^* \gamma_n,
\end{align}
where $d_s^*$ is a measure of model complexity, and $\gamma_n$ is a sequence of penalties on the complexity of the model. In (\ref{eq:clbic}), the first term is the pseudo-loglikelihood, which reflects the goodness-of-fit for a
given model $s$ jointly assessed among multiple data sources, while the second term is the penalty for model
complexity, which enforces sparsity on any model selected.

Let $\theta_T^*$ denote the true value of the parameter under the true model $T.$ Under model $s,$ the parameter space is denoted as $\Theta_s.$ Define $\theta_s^*=\text{argmax}_{\theta \in \Theta_s}E_{\theta_T^*}\{\ell_I(\theta)\},$ under the assumption that such a maximizer is unique in the interior of $\Theta_s.$ We define  the effective degrees of freedom
$d_s^*=\text{tr}\{H_s^{-1}(\theta_s^*)V_s(\theta_s^*)\},$ where $H_s(\theta_s^*)$ and $V_s(\theta_s^*)$ are computed under model $s.$  The term $d_s^*$ has been used to measure model
complexity in many pseudolikelihood settings (Varin \& Vidoni,
2005).

Most consistency results for model selection criteria have been established for bounded model $T$ (Chen \& Chen, 2008; Gao \& Song, 2010) or a divergent true model  for linear regression (Zhang \& Shen, 2010; Kim et al., 2012). The results have been proved based on the exponential decay rate of chi-square statistics. The exponential decay rate is essential for overall selection consistency, as there are exponentially many $p_n^{s_n}$ competing models. By the Bonferroni inequality, we have an upper bound for the overall selection error, which is the sum of all the tail probabilities. If the penalty term $\gamma_n$ is chosen appropriately so that the tail probabilities are exponentially small, then the overall selection error will converge to zero.

Unlike the setting of linear regression, pseudolikelihood-type statistics asymptotically follow a weighted chi-square distribution. It is difficult to obtain a bound of the tail probability at a given sample size $n$ using the limiting distribution. Instead of relying on the limiting distribution, we obtain the tail probability based on their exact distributions. Our approach consists of two steps: first to show that differences in pseudo-loglikelihoods between two competing models $s$ and $T$ can be approximated by quadratic forms and the approximation errors are uniformly bounded across the model space; and second, based on the quadratic forms, apply a large deviation result (Spokoiny \& Zhilova, 2013) to quantify the size of the penalty $\gamma_n$ so that the tail probabilities are exponentially small.

Let $\psi$ denote a random vector, $B$ denote a matrix, and $|| B\psi||^2$ denote a quadratic form. Large deviation results for quadratic forms $|| B\psi||^2$ were established by Spokoiny \& Zhilova (2013) under the exponential moment condition
$$
\text{log}[ E\{ \text{exp}(t\trans \psi)\}]\leq || t||^2/2, \,\,|| t||\leq g.
$$
Here $g$ is a positive constant which differs between Gaussian and non-Gaussian type deviation bounds.
We first prove that such an exponential moment condition can be satisfied asymptotically by sample mean types of statistics, if the original random variables satisfy a cumulant boundedness condition.

\begin{definition}
For a random vector $Z$ of dimension $m,$ let $g(t)$ denote its cumulant generating function, where $t$ denotes an $m$-dimensional real vector. Assume the first two derivatives of
its cumulant generating function satisfy $\vert\partial g/\partial t_j(0)\vert\leq C_1,$ and $\vert\partial^2 g/\partial t_j \partial t_k (0)\vert\leq C_2.$ Assume further that there exists a constant $\delta$ such that with
$|| t||\leq \delta,$ the absolute value of all the third derivatives of
its cumulant generating function satisfy $\vert\partial^3 g(t)/\partial t_j \partial t_k \partial t_l (t)\vert\leq C_3.$
\end{definition}

\begin{lemma}
\label{expmomentmult}
Let $Z_1,\ldots,Z_n$ be independently distributed random vectors of dimension $m$ with zero mean and identity covariance matrices, and let $\eta=n^{-1/2} \sum_i Z_i.$ If each random vector $Z_i$ satisfies the cumulant boundedness condition  with the same bounds, and $s_n^4 \log (p_n)=o(n),$ then $\log [ E\{\exp(t\trans \eta)\}]\leq a^2 ||t||^2/2$ for $|| t||< \{s_n^2 \log (p_n)\}^{1/2}$ and some constant $a^2>1$, when $n$ is sufficiently large.
\end{lemma}

This implies that if the cumulant boundedness condition in Definition 1 holds, we will be able to apply large deviation results to the pseudolikelihood ratio type of statistics arising in our analysis. Next we assume the cumulant boundedness conditions for the derivatives of the pseudo-loglikelihood. We also make assumptions about the distances between the true null model and the competing models. 

\begin{assumption} \label{Assumption3} Assume that all the pseudo-loglikelihoods and their first and second derivatives,  $\ell_I(\theta_s^*;Y_{(i)}),$ $\ell_I^{(1)}(\theta_s^*;Y_{(i)})$, and $\ell_I^{(2)}(\theta_s^*;Y_{(i)})$ satisfy the cumulant boundedness condition in Definition 1 uniformly for all models $s\in \mathcal{S}.$ Also, assume that there exists an neighborhood $|| \theta_s-\theta_s^*||\leq \delta,$ and all the third derivatives of the pseudo-loglikelihoods $ \ell_I^{(3)}(\theta_s;Y_{(i)})$ in that neighborhood satisfy the cumulant boundedness condition in Definition 1 uniformly and that $H_s(\theta_s)$ and $V_s(\theta_s)$ have eigenvalues bounded away from zero and infinity uniformly.
\end{assumption}

Consider generalized linear models with densities from an exponential family. Assume the link function is three times continuously differentiable, all the absolute values of the covariates are uniformly bounded by a constant $b,$ and the linear predictors are bounded. Then Assumption 3 is satisfied. For example, if the density is binomial and the canonical link is used, the boundedness of $X_{ki}^\T \theta_{k}$ ensures that $\mu=\exp( X_{ki}^\T \theta_{k})/\{1+\exp(X_{ki}^\T \theta_{k})\}$ is bounded away from $0$ and $1.$  Let $\partial \log f_k(Y_{ki};\theta_k)/\partial \theta_{k}$ be formulated as in Equation (\ref{binom}). Then the third derivative of its cumulant generating function $g(t)$ 
is bounded by
\begin{align}
\begin{split}
&|\partial^3 g(t)/\partial t_j \partial t_l \partial t_m|\\
=&|\mu (1-\mu) \exp(t^\T X_{ki})X_{kij}X_{kil}X_{kim}\{1-\mu-\mu \exp(t^\T X_{ki})\}/\{1-\mu+\mu \exp(t^\T X_{ki})\}^3|\\
< &(3/4) b^3 \max\{1/\mu^3, 1/(1-\mu)^3\}< \infty.
\end{split}
\end{align}

\begin{assumption} \label{Assumption4} Assume that $s_n^4 \log (p_n) =o(n).$ Define the pseudo Kullback--Leibler distance between the true model $T$ and the competing model $s$  as $E_{\theta_{T}}\{\ell_I (\theta_T;Y_{(i)})-\ell_I (\theta_s^*;Y_{(i)})\}.$ Assume that $\liminf_n \min_{s\in S_{-}} n^{1/2}E_{\theta_{T}}\{\ell_I (\theta_T;Y_{(i)})-\ell_I (\theta_s^*;Y_{(i)})\}/\{s_n \log (p_n)\}^{1/2}=\infty.$
\end{assumption}

This is an assumption regarding the identifiability of the underlying true model. It allows the pseudo Kullback--Leibler distance between the true model and the competing model to tend to zero at a certain rate.

Next we introduce some notation. For any over-fitting model $s$, define a matrix
$D_s=(I_{d_T},0_{d_T,d_s-d_T}),$ with $I_{d_T}$
being an identity matrix of dimension $d_T\times d_T,$ and
$0_{d_T,d_s-d_T}$ denoting a matrix of zeros with
dimension $d_T\times (d_s-d_T).$  For every model $s,$ let the score vector be denoted by $U_n(\theta_s;Y)=\partial \ell_I(\theta_s,Y)/\partial \theta_s,$ and construct the quadratic form  $Q_s=n^{-1} U_n(\theta_{s}^*)\trans
H_{s}(\theta_{s}^*)^{-1}U_n(\theta_{s}^*).$ According to Lemma A1 and Lemma A2 in the Appendix, we have
$2 \{\ell_I(\widehat{\theta}_s)-2 \{\ell_I(\widehat{\theta}_T)\}=(Q_s-Q_T)\{1+o_p(1)\}=Q_{s/T}\{1+o_p(1)\},$ with $Q_{s/T}=U_s(\theta_{s}^*)\trans M_{s/T}U_s(\theta_{s}^*),$ where $M_{s/T}$ denotes the
difference matrix $H_{s}(\theta_{s}^*)^{-1}-D_s\trans
H_T^{-1}(\theta_{T}^*)D_s.$ Define $B_s=V_s^{1/2}(\theta_s^*)M_{s/T}V_s^{1/2}(\theta_s^*).$ It can be shown that $\text{tr}(B_s)=d_s^*-d_T^*.$ Denote $\tau=\lambda_{max}(B_s),$  $\overline{\tau}=\text{tr}(B_s)/(d_s-d_T),$ and $\omega=\max_{s\in \mathcal{S}} \tau/\overline{\tau}.$ For the true loglikelihood, $\omega=1.$

We now establish a consistency result for the pseudolikelihood information criterion for  unbounded true model size. 
\begin{theorem}
\label{thm3} Define $\gamma_n=6\omega(1+\gamma)\log (p_n)$ for some $\gamma>0$ or $\gamma_n=6\omega\{\log (p_n) +\log \log (p_n)\}$. Under regularity Conditions 1--3 and Assumptions 1--4, as $n\rightarrow
\infty$,
$$
{\rm pr}\bigl\{\min_{s\in \mathcal{S}}\text{pseu-BIC}(s)>
\text{pseu-BIC}(T)\bigr \} \rightarrow 1. 
$$
\end{theorem}

Theorem \ref{thm3} demonstrates that, with appropriate penalty term, the \BIC\ type of information criterion based on compounded marginal likelihoods from different sources can be selection consistent, even if the underlying true model size tends to infinity. This result  includes the usual \BIC\ based on the true likelihood as a special case with $\omega=1,$ and $\gamma_n=6(1+\gamma)\log p_n.$ Comparing to the result of Chen and Chen (2008), which  proved the consistency of extended \BIC\ with the true model size being bounded by a constant,  Theorem \ref{thm3} is the first result to establish the selection consistency of a \BIC\ type of information criterion with unbounded true model size. In Section \ref{sec5}, we use the proposed pseudolikelihood information criterion to select the optimal tuning parameter for group penalization. 

\section{SIMULATIONS}\label{sec5}

\subsection{Continuous responses}\label{sec5.1}
 For our first simulation, we generated four different types of experiments, $K=4.$ For each experiment, a continuous response $Y_{ki}$ and the associated covariates $X_{ki}=(x_{ki1},\ldots, x_{kip_n})$ were observed for subject $i$. We took  sample sizes $n=500,$ and $1000,$ and the number of covariates $p_n=200$ and $p_n=1000.$ For different experiments, the regression covariates were different.  The number of true covariates was set to be $q_n=50.$ For $j=1,\ldots,q_n$, $\theta_{kj}$ was drawn from the uniform distribution on $(0.05, 0.5),$ whereas for $j=q_n+1,\ldots,p_n$, $\theta_{kj}$ was set to be zero. The covariates $X_{ki},$ were partitioned into independent blocks of 50 covariates, and within each block, the $50$ covariates were simulated from the multivariate normal distribution with variances equal to 1 and off-diagonal covariances all equal to 0.2. For each experiment, the mean parameter is $\mu_{ki}=X_{ki}\trans\theta_{k}.$ We simulated $Y_{i}$ from a multivariate normal distribution with mean $\mu_i=(\mu_{1i},\ldots,\mu_{Ki}),$ and covariance matrix $\Sigma.$ The covariance matrix was compound symmetric with variances equal to 1 and off-diagonal covariances equal to 0.7. 

We used the group smoothly clipped absolute deviation penalty function to perform feature selection and used the pseudolikelihood information criterion to select the tuning parameters. For group penalized estimation, we used the group descent algorithm proposed by Breheny \& Huang (2015).  With regard to the penalty term, Theorem 3 provides a theoretical value of  $6\omega(1+\gamma)d_s^* \log (p_n)$, which leads to consistent model selection when the sample size increases to infinity. Here the effective degrees of freedom is estimated as $\widehat{d}_s^*=\text{tr}(\widehat{H}_s^{-1}\widehat{V}_s),$ where $\widehat{H}_s$ is estimated as the observed Hessian matrix, and $\widehat{V}_s$ is estimated as the sample covariance matrix of the composite scores. We thus set the penalty term to  be $c \widehat{d}_s^*\log (p_n),$ where $c$ is a constant factor. This penalty term thus has the same asymptotic order as the theoretical penalty term. We set the constant factor $c$ to values 1 or 6 and examined how the sensitivity and selectivity of our method changes with the size of $c.$  Table \ref{tab2} provides the positive selection rates and false discovery rates of our data integration method and the single experiment analysis based on the first experiment only.  When $c$ changes from 1 to 6, the data integration method's positive selection rate and false discovery rate decrease slightly. A large improvement in the performance of our data integration method is observed compared to single experiment analysis.  For example, when $c=1$, $n=500$, $p_n=1000,$ the positive selection rate and false discovery rate of data integration method are 1.00 and 0.02, respectively, whereas those of the single experiment analysis are 0.81 and 0.35, respectively. 

\subsection{Continuous responses with correlations between predictors and non-predictors }\label{sec5.2}

We investigated the performance of the proposed method with the group  smoothly clipped absolute deviation penalty and the group lasso penalty when there exist correlations between predictors and non-predictors. This setting violates the strong-irrepresentability conditions, and thus affects the performance of the lasso penalty.  We generated four different types of experiments.  The covariates $X_{ki}$ were partitioned into independent blocks of 200 covariates. The first block contains 50 true predictors and 150 non-predictors. These 200 covariates were simulated from the multivariate normal distribution with variances equal to 1 and off-diagonal covariances equal to 0.2 or 0.5. For the remaining non-predictors, they were simulated from independent normal distribution with variances equal to 1. All other parameter settings are the same as in Section \ref{sec5.1}. We chose the sample size $n=1000$ and $p=1000.$ Table \ref{tab6} show that in the presence of correlation between predictors and non-predictors, the group smoothly clipped absolute deviation outperforms group lasso with smaller false positive discovery rate and smaller sum of squared errors.  

\subsection{Mixtures of continuous and binary responses}\label{sec5.3}
Our third simulation examines the performance of our method on data with correlated continuous and binary responses. We generated four different types of experiments, $K=4.$ All the experiments share the same set of covariates $X_{i}=(x_{i1},\ldots, x_{ip_n})$ for subject $i$.  We took  sample sizes $n=1000,$ and $1500,$ and the number of covariates $p_n=200$ and $p_n=1000.$ For different experiments, the regression covariates are different.  The number of true covariates was set to be $q_n=50.$ For $j=1,\ldots,q_n$, $\theta_{kj}$ was drawn from the uniform distribution on $(0.05, 0.5),$ whereas for $j=q_n+1,\ldots,p_n$, $\theta_{kj}$ was set to be zero. The covariates $X_{i},$ were simulated from a normal distribution with mean zero and variance equal to 1. For each experiment, the mean parameter is $\mu_{ki}=X_{i}^\T\theta_{k}.$ We simulated $Y_{i}^*$ from a multivariate normal distribution with mean $\mu_i=(\mu_{1i},\ldots,\mu_{Ki}),$ and covariance matrix $\Sigma.$ The covariance matrix was compound symmetric with variances equal to 1 and off-diagonal covariances equal to 0.7.  For the first two experiments, the observed responses are continuous values $Y_{ki}=Y_{ki}^*,$ $k=1,\,2;$ for the third and fourth experiments, the actual observed binary data are the dichotomized version of the continuous observations $Y_{ki}=I(Y_{ki}^*>0),$ $k=3,\,4.$ Table \ref{tab4} provides the performance of the data integration method for the correlated binary and continuous data and when $c=1,6.$ The result is consistent with Table \ref{tab2}, where the data integration method outperforms the single experiment analysis based on the first experiment only. For example, when $n=1000,$ $p_n=1000,$ $c=1,$ the positive selection rate and false discovery rate of our data integration method are 0.99 and 0.01, respectively, whereas those of  the single experiment analysis are 0.90 and 0.34, respectively. 

\begin{table}
\begin{center}
 \caption{\baselineskip=12pt Performance of the data integration method compared to the single experiment analysis for multivariate normal responses.}
\label{tab2}
\begin{tabular}{ccrrrrrrrrr}
 p & n & DI & DI & SI &SI &&DI &DI &SI &SI \\
 &   & psr & fdr & psr &fdr &&psr &fdr &psr &fdr \\
&   &  \multicolumn{4}{c}{$c=1$} && \multicolumn{4}{c}{$c=6$}   \\
200 &   500 & 100\% & 2\% & 91\% & 28\% && 99\% & 0\% & 73\% & 3\%  \\ 
 &   std& 1\% & 2\% & 5\% & 8\% && 1\% & 0\% & 13\% & 4\%   \\ 

200  &  1000 & 100\% & 0\% & 96\% & 27\% && 100\% & 0\% & 90\% & 4\%  \\ 
 &    std& 0\% & 1\% & 3\% & 8\% && 1\% & 0\% & 6\% & 3\%   \\ 

1000  &  500 & 100\% & 7\% & 81\% & 35\% && 99\% & 0\% & 57\% & 2\%  \\ 
 &    std& 1\% & 7\% & 7\% & 10\% && 1\% & 1\% & 13\% & 3\%  \\ 

1000  &  1000 & 100\% & 0\% & 91\% & 29\% && 100\% & 0\% & 81\% & 4\%  \\ 
 &    std& 0\% & 1\% & 4\% & 8\% && 1\% & 0\% & 7\% & 3\%  \\ 

\end{tabular}%
\begin{tabnote}
DI, data integration method; SI, single experiment analysis;
psr, positive selection rate in percent;
fdr, false discovery rate in percent;
the reported numbers are average psr and fdr from 100 simulated data sets;
std, the sample standard deviation of psr and fdr from 100 simulations;
c, the free multiplicative constant for the penalty.
\end{tabnote}
\end{center}
\end{table}

\begin{table}
 \caption{\baselineskip=12pt Comparison of group lasso and group smoothly clipped absolute deviation (scad) penalty in the presence of correlated covariates with $n=1000$ and $p=1000$.}
\label{tab6}
\begin{center}
\begin{tabular}{crrrrrr}
 r &psr & fdr & sse & psr & fdr &sse  \\
 &   \multicolumn{3}{c}{lasso} & \multicolumn{3}{c}{scad}   \\
 0.20 & 100\% & 1\% & 224 & 100\% & 0\% & 47 \\ 
std & 1\% &  1\%  &  82 & 1\% & 1\%& 7\\ 
 0.50 & 99\% & 3\% & 464 & 99\% & 1\% & 282 \\ 
std& 2\% & 5\%  &  106 & 2\% & 3\%& 118\\ 
\end{tabular}%
\begin{tabnote}

psr, positive selection rate in percent;
fdr, false discovery rate in percent; sse, the sum of squared errors of the penalized estimate $||\hat{\theta}-\theta||_2^2$;
the reported numbers are average values from 100 simulated data sets;
std, the sample standard deviation of psr, fdr and sse computed from 100 simulations;
r, the correlation between true predictors and false predictors; the columns of sse have been multiplied by 100. 
\end{tabnote}
\end{center}
\end{table}

\begin{table}
 \caption{\baselineskip=12pt Performance of the data integration method compared with single experiment analysis for multivariate mixed binary and continuous responses, for the case that the binary and continuous responses are correlated. The free multiplicative constant $c$ for the penalty is $c = 1,6$. }
\label{tab4}
\begin{center}
\begin{tabular}{ccrrrrrrrrr}
 p  & n & DI & DI & SI &SI &&DI &DI &SI &SI \\
 &   & psr & fdr & psr &fdr &&psr &fdr &psr &fdr \\
&   &  \multicolumn{4}{c}{$c=1$} && \multicolumn{4}{c}{$c=6$}   \\
200  &  1000 & 100\% & 1\% & 96\% & 30\% && 93\% & 0\% & 89\% & 5\%  \\ 
 &    std& 1\% & 2\% & 3\% & 8\% && 13\% & 0\% & 6\% & 3\% \\ 

200  &  1500 & 100\% & 1\% & 99\% & 29\% && 97\% & 0\% & 94\% & 5\%  \\ 
 &    std& 0\% & 1\% & 2\% & 7\% && 3\% & 0\% & 4\% & 4\% \\ 

1000  &  1000 & 99\% & 1\% & 90\% & 34\% && 83\% & 0\% & 81\% & 5\%  \\ 
 &    std& 1\% & 1\% & 5\% & 7\% && 26\% & 0\% & 7\% & 4\% \\ 

1000  &  1500 & 100\% & 1\% & 95\% & 32\% && 96\% & 0\% & 88\% & 4\%  \\ 
 &    std& 1\% & 3\% & 4\% & 7\% && 3\% & 0\% & 6\% & 4\%   \\ 
\end{tabular}%
\begin{tabnote}
DI, data integration method; SI, single experiment analysis;
psr, positive selection rate in percent;
fdr, false discovery rate in percent;
the reported numbers are average psr and fdr from 100 simulated data sets;
std, the sample standard deviation of psr and fdr from 100 simulations;
c, the free multiplicative constant for the penalty.
 
\end{tabnote}
\end{center}
\end{table}

\section{Data Analysis}\label{sec6}

First we applied our method to Example \ref{Example1} discussed in the introduction. The data consists of two different microarray experiments on breast cancer cells
(Wang, et al., 2005; Iwamoto et al., 2011). In the first experiment, the gene expression profiles from total RNA were obtained from frozen tumor samples from  lymph-node-negative patients who had not received adjuvant systemic treatment. In the second experiment, pre-treatment fine-needle aspirations from primary tumors were obtained and RNA was extracted and hybridized to microarrays. Due to the different experiment protocols, the two sets of gene expression profiles are globally different. Both experiments were conducted to study the difference of the gene expression profiles between estrogen-receptor positive and estrogen-receptor negative patients. Understanding the genetic difference between the two clinically important subclasses can lead to more efficient treatments tailored to individual patients. The training data set consists of a total of 170 samples with 35 samples from the estrogen-receptor positive patients and 50 samples from the estrogen-receptor negative patients. In Figures 1(a) and 1(b), the heatmaps of the two experiments are shown. The objective of the analysis is to combine the data from the two experiments and find a common set of candidate genes that can be used to classify the estrogen-receptor positive and estrogen-receptor negative cases.  For each of the experiments, we built a logistic model with the two subclasses as the binary responses and the expressions levels of all the genes as the covariates. We applied our integrative analysis method and used the group smoothly clipped absolute deviation penalty to penalize the regression coefficients. With increasing penalty size, we obtained the solution path. Figures 1(c) and 1(d) depict the selected genes when the candidate list decreases to four candidates. The selected top candidates exhibit consistent significant differential behavior in both experiments. The logistic models based on the selected four covariates were used to classify the subclasses of a different validation data set, which contains 13 samples from the first experiment and 54 samples from the second experiment. Among all the 67 validation samples,  16 samples were misclassified. The overall accuracy rate of the classification on the validation data was 76 percent.

\begin{figure}
\label{fig2}

 \hspace{-4mm} 
    \begin{minipage}{0.5\linewidth}
         \centering    
        \includegraphics[width=3.0in]{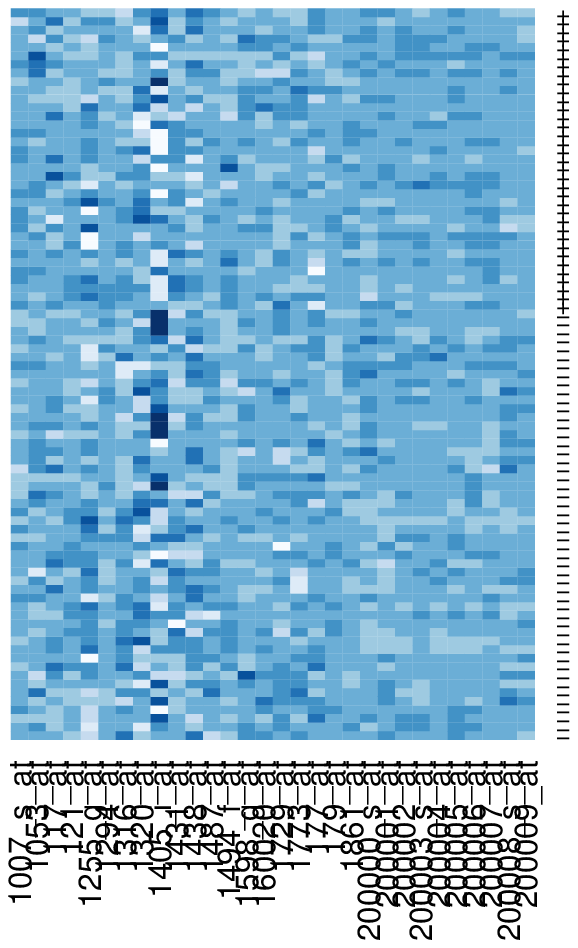}\\
       (a)
    \end{minipage}  
    \begin{minipage}{0.5\linewidth}
        \centering        
        \includegraphics[width=3in]{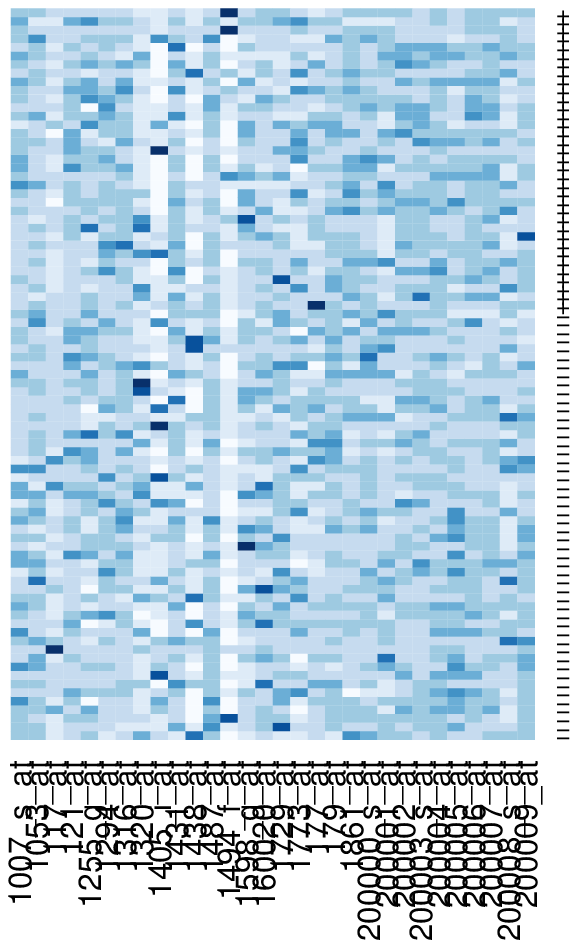}\\
        (b)
    \end{minipage}

    \begin{minipage}{0.5\linewidth}
        \centering
        \includegraphics[width=3in]{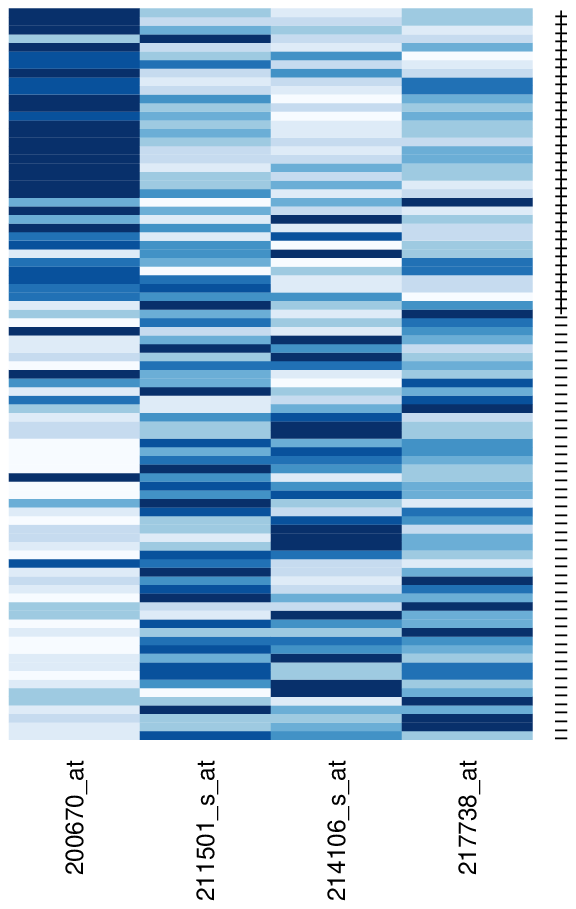}\\
        (c)
    \end{minipage}
     \begin{minipage}{0.5\linewidth}
        \centering
        \includegraphics[width=3in]{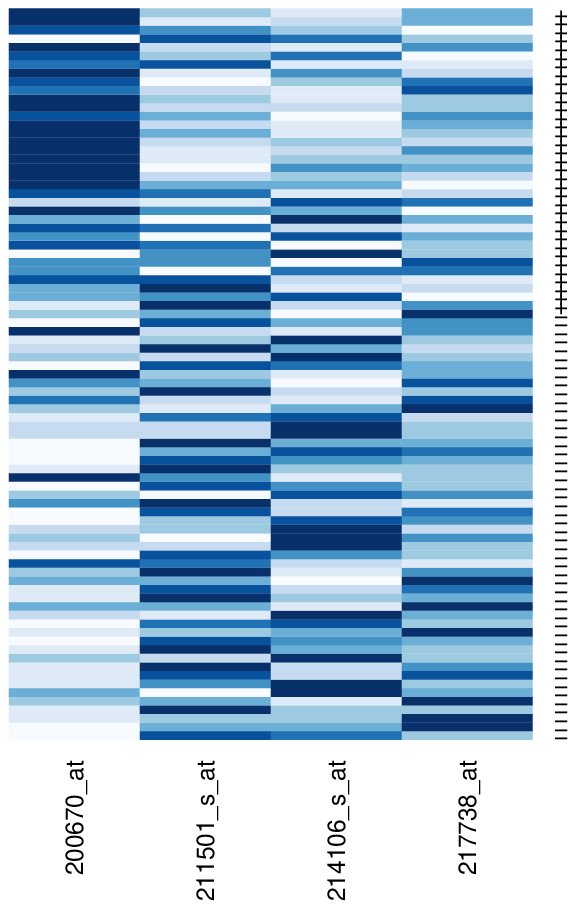}\\
        (d)
    \end{minipage}


\caption{Images (a) and (b) are heatmaps of the microarray gene expression profile from the two experiments in Wang et al. (2005) and in Iwamoto et al. (2011), respectively. Only the first 30 genes are depicted in the heatmaps; Images (c) and (d) are the heatmaps of the gene expression levels of the selected fours genes from  both experiments; "+" denotes estrogen-receptor positive samples and "-" denotes estrogen-receptor negative samples.}

\end{figure}

Second, we applied our method to Example \ref{Example2} discussed in the introduction. 
The data set contains financial market indices. We are interested in a panel of three indices including the S\&P 500 Index, the Dow Jones Index and the VIX index. The VIX is a measurement of implied volatility of the S\&P 500 Index and is highly negatively correlated with the S\&P 500 Index. The S\&P Index and the Dow Jones Index are positively correlated. The 46 covariates are the major international equity indices, the North American bond indices, and the major commodities indices. For example,  the Nikkei 225 index is a benchmark of the Japanese equity market, the CBOE 10-Year Treasury Note is a US bond market benchmark, and the Philadelphia Gold and Silver Index is an index of thirty precious metal mining companies that are traded on the Philadelphia Stock Exchange. The goal of the analysis is to select a subset of  covariates to model the panel of the S\&P 500 Index, the Dow Jones Index and the VIX index. The training data set consists of the three-year market performance of the S\&P 500 Index, Dow Jones Index and VIX index and the 46 covariates. The data contains the index values from every three days between September, 2013 to June, 2016.  For each index, the value used in the analysis is $\log$ (today's value/yesterday's value)*100. Because of the three-day spacing, there are a total of 232 records. The transformed values of the S\&P 500 Index, the Dow Jones Index and the VIX index are not autocorrelated at a 5\% significance level. For each response, we constructed a linear regression model based on the same set of covariates. 

We applied both the group \text{lasso} penalty and the group smoothly clipped absolute deviation penalty.  We used the pseudolikelihood information criterion to determine the optimal penalty term. Figure 2(a) shows the pseudolikelihood information criterion curves of the solution paths selected by both the group \text{lasso} and the group smoothly clipped absolute deviation penalty. The subset selected by group \text{lasso} contains 37 covariates and the subset obtained by the group smoothly clipped absolute deviation penalty contains 34 covariates. The two methods identified 31 covariates in common. In order to validate the submodels, we used the model built from the training data set to perform prediction on a different validation data set of 232 records. The prediction sum of squared errors for the submodel selected by the group smoothly clipped absolute deviation penalty, the submodel selected by the group \text{lasso} penalty, the full model and the total sum of squared variation in the response are 4930, 5746.81, 6375.95 and 12539.58 for the VIX index; 2.10, 2.57, 1.62 and 165.24 for the S\&P 500 index; and 15.27, 14.81, 17.47 and 160.83 for the Dow Jones Index, respectively. It is evident that both selected submodels have small prediction errors compared to the total variation in the responses across all the three responses in the validation data set. The submodel selected by the group smoothly clipped absolute deviation penalty has smaller prediction errors than the one selected by the group \text{lasso} in two out of the three responses. Figures 2(b), 2(c) and 2(d) depict the observed responses of the three indexes and the prediction curves provided by the submodel selected by the group smoothly clipped absolute deviation penalty. 

\begin{figure}
\label{fig2}
 \hspace{-4mm} 
    \begin{minipage}{0.5\linewidth}
            
        \includegraphics[width=3.0in]{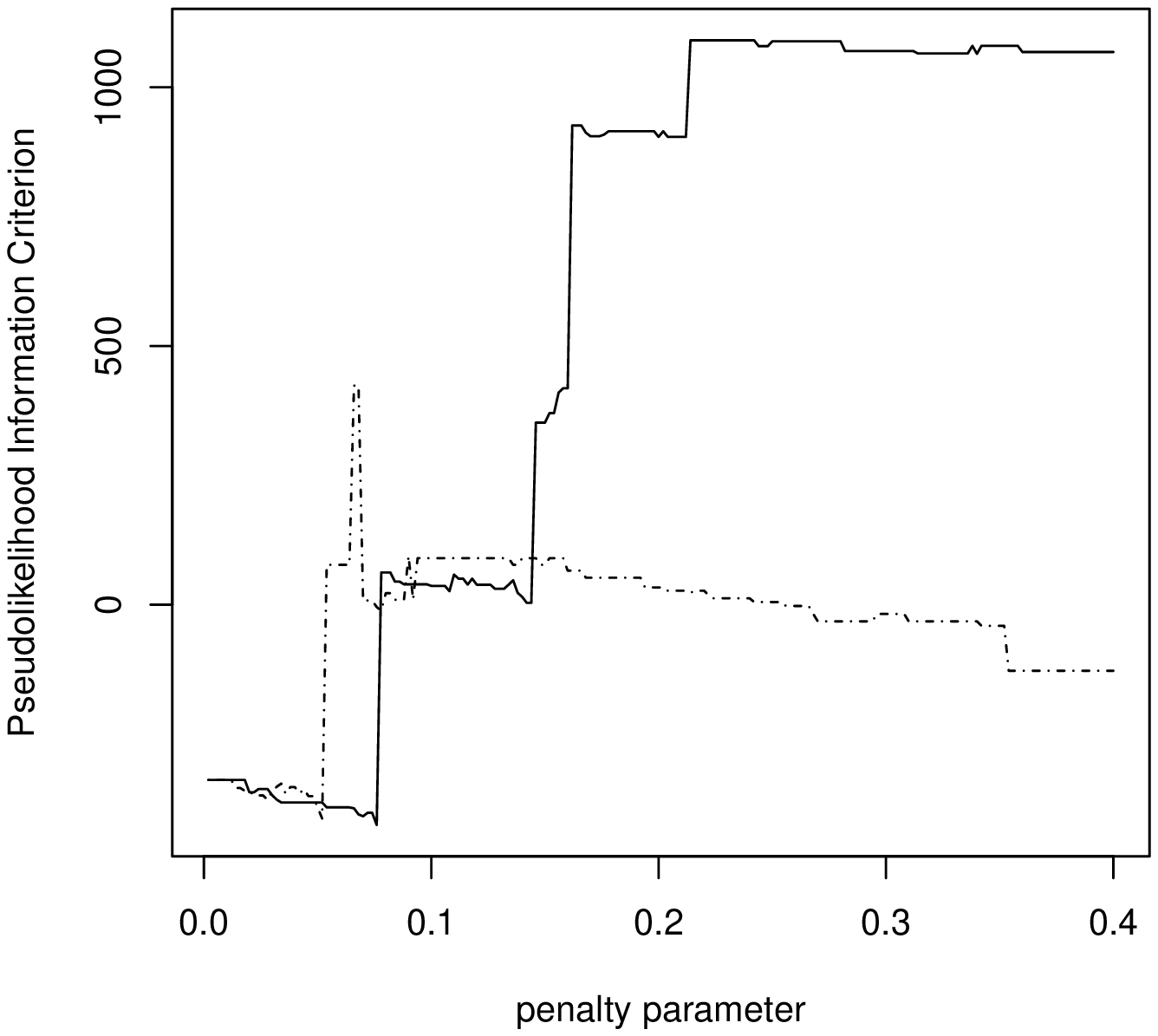}\\
       
    \end{minipage}  
    \begin{minipage}{0.5\linewidth}
        \centering        
        \includegraphics[width=3in]{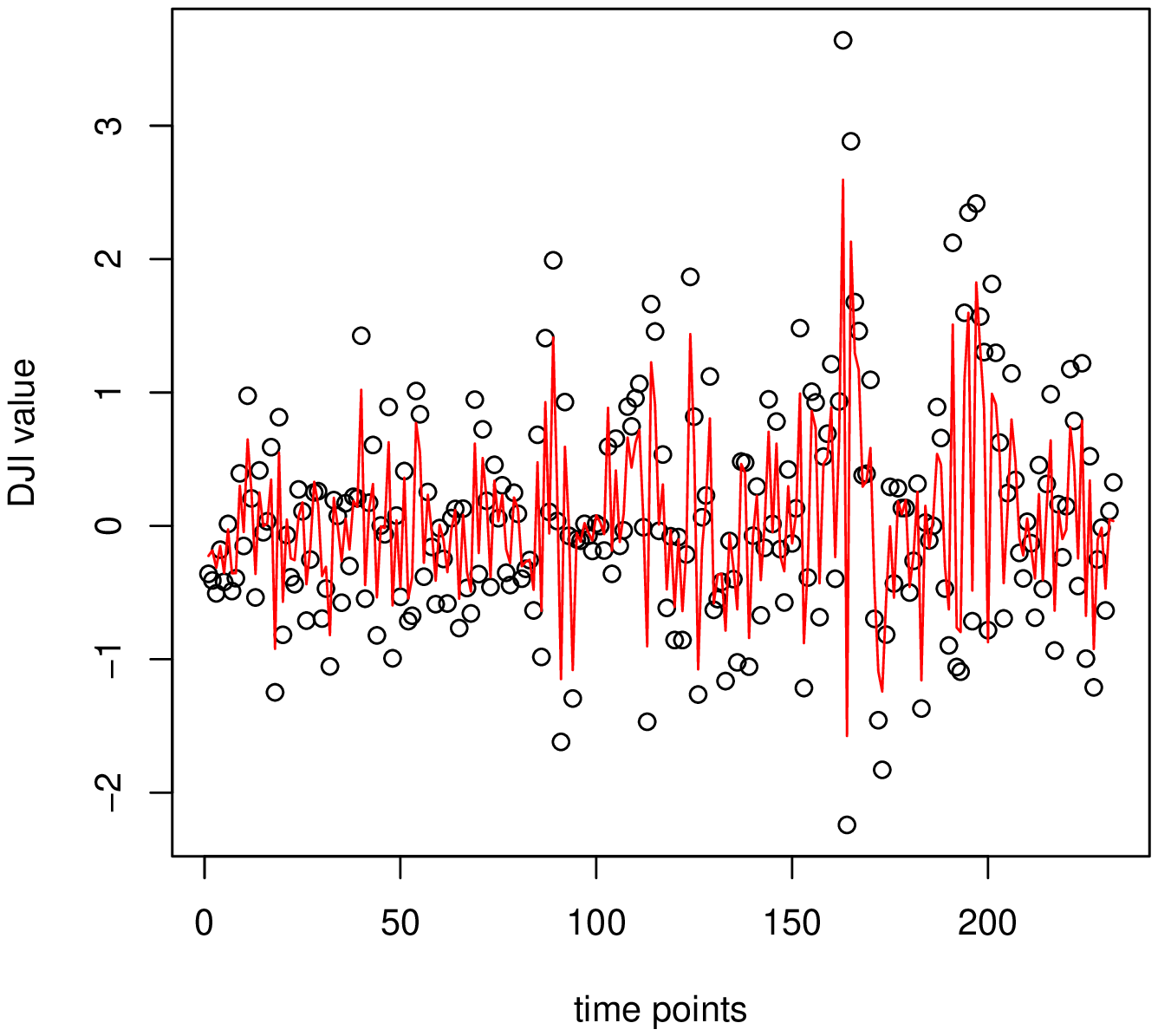}\\
        (a)
    \end{minipage}
    \begin{minipage}{0.5\linewidth}
        \centering
        \includegraphics[width=3in]{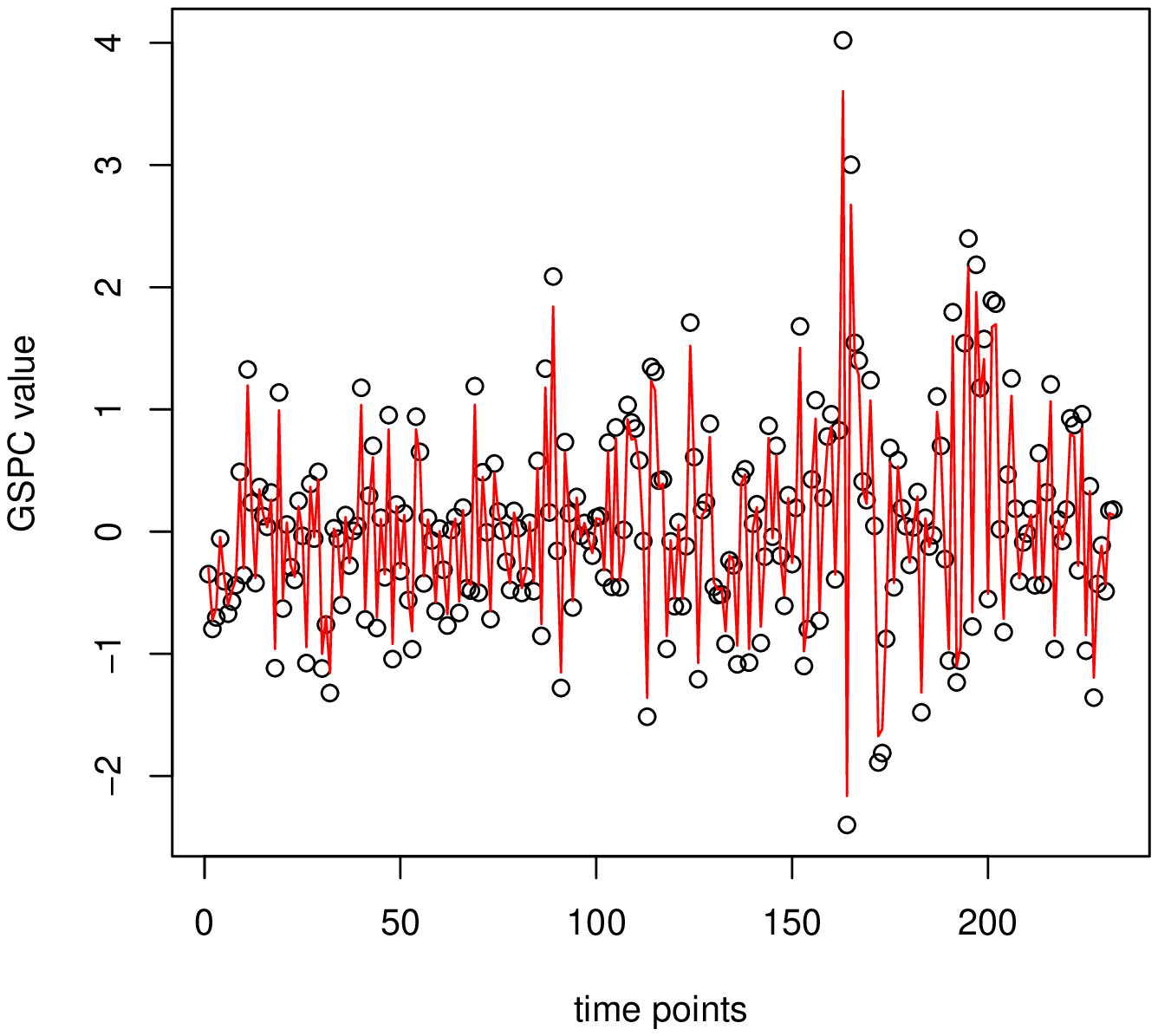}\\
        (b)
    \end{minipage}
     \begin{minipage}{0.5\linewidth}
        \centering
        \includegraphics[width=3in]{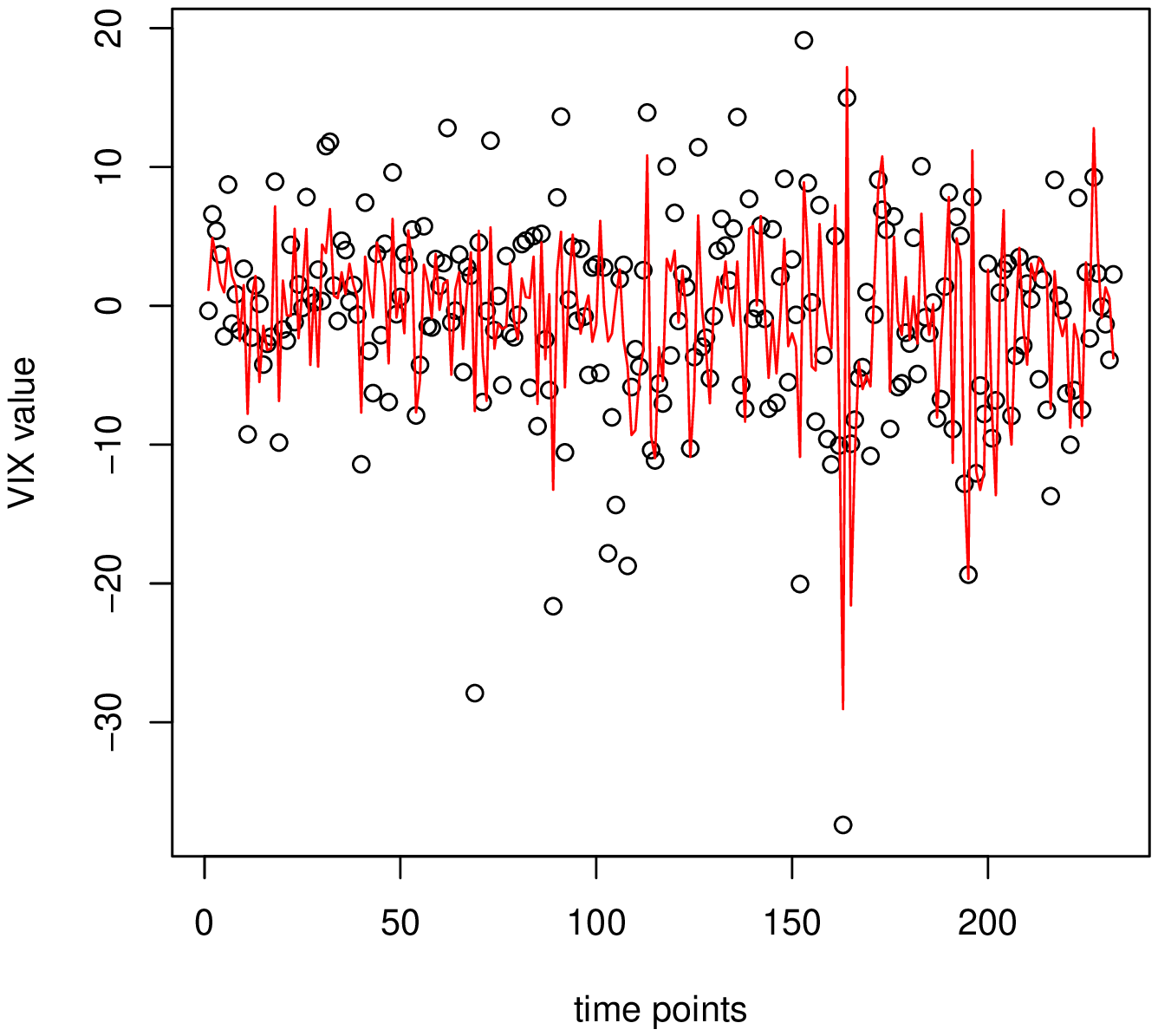}\\
        (c)
    \end{minipage}


\caption{Pseudolikelihood information criterion curve with the solid line indicating the solution path selected by the group smoothly clipped absolute deviation penalty, with the dashed line indicating the solution path selected by the group \text{lasso} (a); prediction of the Dow Jones Index (DJI) on the validation data (b); prediction of the S\&P 500 (GSPC) index on the validation data (c); prediction of the VIX index on the validation data (d).  In (b), (c) and (d), circles represent the observed values, solid lines represent the prediction by the subset model selected by the group smoothly clipped absolute deviation penalty. }

\end{figure}


\section*{Acknowledgments}
Gao's research was supported by a discovery grant from the Natural Sciences and Engineering Research Council of Canada. Carroll's research was supported by a grant from the National Cancer Institute. Carroll is also Distinguished Professor, School of Mathematical and Physical Sciences, University of Technology Sydney, Australia.

\appendix


\appendixone
\section*{Appendix 1}
\subsection*{Regularity conditions for likelihood inference}


\noindent {\it Condition 1}   Assume that the pseudo-loglikelihood admits  third derivatives for almost all $Y$ and for all $\theta\in \B,$ where the open set $\B\subset \Theta$,  contains the true $\theta^*.$ Furthermore,
$\vert\partial^3 \ell_I(\theta;Y_{(i)})/\partial \theta_{j} \partial \theta_{l} \partial \theta_{m}\vert<W_{jlm}(Y_{(i)}),
$ for $\theta \in \B,$ and for all $j, l, m \in \{vw, v=1,\ldots,K, w=1,\ldots,p_n\},$ where $
E_{\theta^*}\{W_{jlm}(Y_{(i)})\}^{2\kappa}\leq M_4
$ for an interger $\kappa\geq 1.$
\vskip 3mm

\noindent {\it Condition 2} The parameter space $\theta \in \Theta$ is a closed set. Each density $f_k(Y_k;\theta_k)$ is a measurable function of $Y_k$ for any $\theta_k,$ and is distinct for different values of $\theta_k.$ Let $\theta^*$ denote the true value of $\theta.$ We assume that
$E_{\theta^*}\{\partial \log f_k(Y_{ki};\theta_k)/\partial \theta_{kj}   \}=0,$ and 
\begin{eqnarray*}
E_{\theta^*}\left[\frac{\partial^2 \log \{f_k(Y_{ki};\theta)\}}{\partial \theta_{kj} \partial \theta_{kl}}\right]   =-E_{\theta^*}\left[\frac{\partial \log \{f_k(Y_{ki};\theta_k)\}}{\partial \theta_{kj} } \frac{\partial \log \{f_k(Y_{ki};\theta_k)\}}{\partial \theta_{kl} }  \right]
\end{eqnarray*} for
$j,l =1,\dots,p_n$ and $k=1,\dots,K$. 
\vskip 1mm
\noindent {\it Condition 3}   Let the submatrices of $H(\theta^*)$ and $V(\theta^*)$ with respect to the parameters in $\theta_a$ be denoted as $H^{(1)}(\theta^*)$ and $V^{(1)}(\theta^*).$ Assume that $
0<\lambda_{\min}\{H^{(1)}(\theta)\}<\lambda_{\max} \{H^{(1)}(\theta)\} <\infty,
$ and $
0<\lambda_{\min}\{V^{(1)}(\theta)\}<\lambda_{\max}\{V^{(1)}(\theta)\}<\infty,
$ where $\lambda_{\min}$ and $\lambda_{\max}$ denote the smallest and largest eigenvalues.

\section*{Appendix 2}
\subsection*{Proofs of the theorems}
The proofs will refer to Lemmas A1--A6 that are listed in Appendix 3.
\begin{proof}[of Theorem~\ref{thm1}]
Taking the first derivative of the objective function $Q(\theta)$ with respect to the $j$th grouped parameters $\theta^{(j)}$, we show that $\widehat{\theta}$ satisfies the Karush--Kuhn--Tucker conditions. By the definition of an oracle estimate, for $1\leq j \leq q_n,$ $\partial \ell_I(\theta)/\partial  \theta^{(j)}|_{\widehat{\theta}}=0.$ It can be shown that $\hbox{pr}(\min_{1\leq j \leq q_n}||\widehat{\theta}^{(j)}||\geq a\lambda_n) \rightarrow 1.$ This holds true because
$\min_{1\leq j \leq q_n}||\widehat{\theta}^{(j)}||\geq \min_{1\leq j \leq q_n}||\theta^{*(j)}||-\max_{1\leq j \leq q_n}||\theta^{*(j)}-\widehat{\theta}^{(j)}||,$ $\min_{1\leq j \leq q_n}||\theta^{*(j)}|| > M_5 n^{-(1-c_2)/2},$ $\max_{1\leq j \leq q_n}||\theta^{*(j)}-\widehat{\theta}^{(j)}||=O_p(n^{-(1-c_1)/2})$, and $\lambda_n=o(n^{-(1-c_2+c_1)/2}).$ Thus $\widehat{\theta}^{(j)}$ belongs to the third case in formula (\ref{kkt1}) and $\partial Q(\theta)/\partial \theta^{(j)}|_{\widehat{\theta}}=0.$ 

Let $S_j(\theta)=\partial \ell_I(\theta)/\partial \theta^{(j)}.$ For the remaining parameters, we prove that $\hbox{pr}(\max_{q_n <j \leq p_n}|| S_j(\widehat{\theta})||\leq n\lambda_n)\rightarrow 1.$  For each $k,$ $\widehat{\theta}_k$ is an oracle estimate for $\theta_k.$ Therefore, by Formula (A.7) in Kwon \& Kim (2012), it can be proved that
$
\hbox{pr}\{\max_{q_n< j \leq p_n}| \partial \ell_I(\widehat{\theta})/\theta_{kj}|>n\lambda_n/K^{1/2}\}\rightarrow 0.
$
Hence $\hbox{pr}\{\max_{q_n <j \leq p_n}\vert\vert S_j(\widehat{\theta})\vert\vert> n\lambda_n\}\leq \sum_{k=1}^K \hbox{pr}\{\max_{q_n< j \leq p_n}| \partial \ell_I(\widehat{\theta})/\theta_{kj}|> n\lambda_n/K^{1/2}\} \rightarrow 0.$ Thus $\widehat{\theta}^{(j)}$ belongs to the first case in formula (\ref{kkt1}) and $\partial Q(\theta)/\partial \theta^{(j)}|_{\widehat{\theta}}=0.$ 
\end{proof}

\begin{proof}[of Theorem~\ref{thm2}]
Let $\triangledown \ell_I(\theta)=\partial \ell_I(\theta)/\partial \theta$ denote the score vector of the pseudolikelihood.  Let $\triangledown_1$ denote partial differentiation with respect to $\theta_a.$ Let $\triangledown^2 \ell_I(\theta)$ denote the matrix of second derivatives $\partial^2 \ell_I(\theta)/\partial \theta \partial \theta\trans.$
We expand $\triangledown_1 \ell_I(\widehat{\theta})$ around $\theta^*,$ knowing that $\triangledown_1 \ell_I(\widehat{\theta})=0$, as
$\triangledown_1 \ell_I(\widehat{\theta})= \triangledown_1 \ell_I(\theta^*) + \triangledown_1^2 \ell_I(\theta^*)(\widehat{\theta}_a-\theta_a^*) +
R,$ where $R$ is a $q_n^*\times 1$ vector of remainder terms with $R_i=(1/2)\sum_{j,l}\partial^3 \ell_I(\theta)/\partial \theta_i \partial \theta_j \partial \theta_l|_{\widetilde{\theta}} (\widehat{\theta}_j-\theta^*_j)(\widehat{\theta}_l-\theta^*_l),$ and $i,\,j,\,l\,\in \{st,s=1,\ldots,K, t=1,\ldots,q_n\},$ and $\widetilde{\theta}$ between $\theta^*$ and $\widehat{\theta}.$
This leads to
$
n^{-1} \{\triangledown_1^2 \ell_I(\theta^*)\}(\widehat{\theta}_a-\theta_a^*)=
-n^{-1} \{\triangledown_1 \ell_I(\theta^*)+R\}.
$
By Assumption 1, we have that $|\partial^3 \ell_I(\theta^*)/(n \partial \theta_i \partial \theta_j \partial \theta_k)|\leq n^{-1}\sum_{l=1}^n W_{ijk}(Y_{(l)}).$ Thus
\begin{align*}
\begin{split}
|R_i/n|&\leq n^{-1} \sum_l \sum_j \sum_k W_{ijk}(Y_{(l)})(\widehat{\theta}_j-\theta_j^*)(\widehat{\theta}_k-\theta_k^*)\\
&=n^{-1} \sum_l \sum_j \sum_k \bigl[W_{ijk}(Y_{(l)})-E\{W_{ijk}(Y_{(l)})\}\bigr](\widehat{\theta}_j-\theta_j^*)(\widehat{\theta}_k-\theta_k^*)\\
&\hskip 10mm + n^{-1} \sum_l \sum_j \sum_k E\{W_{ijk}(Y_{(l)})\}(\widehat{\theta}_j-\theta_j^*)(\widehat{\theta}_k-\theta_k^*)=I_1+I_2,
\end{split}
\end{align*}
where $I_2\leq M q_n ||\widetilde{\theta}_1-\theta_1^*||^2=O_p(q_n^2/n)$ by the Cauchy--Schwarz inequality for some constant $M,$ say. Let $W^*_{ijkl}$ denote the centered random variable $W_{ijk}(Y_{(l)})-E\{W_{ijk}(Y_{(l)})\}.$ By the Rosenthal inequality, $E\{(\sum_l W^*_{ijkl})^2\}=O(n).$ Using the Markov inequality, 
$(\sum_l W^*_{ijkl})^2$ $=O_p(n).$ For $I_1$, by the Cauchy--Schwarz inequality,
\begin{align*}
\begin{split}
I_1&\leq n^{-1}  \{\sum_j(\widehat{\theta}_j-\theta_j^*)^2\}^{1/2} \{\sum_k(\widehat{\theta}_k-\theta_k^*)^2\}^{1/2} \{\sum_j \sum_k (\sum_l W^*_{ijkl})^2\}^{1/2}       \\
&=O_p(q_n/n^2) \{\sum_j \sum_k (\sum_l W^*_{ijkl})^2\}^{1/2}=O_p(q_n^2 n^{3/2}).
\end{split}
\end{align*}
By combining these results for $I_1$ and $I_2$, we have that $|R_i|=O_p(q_n^2).$ Let $A_{nr}$ denote the $r$th row of $A_n.$ It then follows that
\begin{align*}
\begin{split}
&|n^{-1/2}A_{nr}  \{V^{(1)}(\theta^*)\}^{-1/2}  R|
\leq  n^{-1/2} || A_{nr}||\lambda_{max}[\{V^{(1)}(\theta^*)\}^{-1/2}]|| R||=O_p\{(q_n^5/n)^{1/2}\}=o_p(1).
\end{split}
\end{align*}
Thus the vector of $n^{-1/2}A_{n} \{V^{(1)}(\theta^*)\}^{-1/2} R$ converges to zero in probability. By Lemma 8 in Fan \& Peng (2004), $||\bigl\{H^{(1)}(\theta^*)+n^{-1} \triangledown_1^2(\theta^*)\bigr\}(\widehat{\theta}_a-\theta_a^*)||\leq o_p(q_n^{-1})O_p\{(q_n/n)^{1/2}\}$, so that
\begin{align*}
\begin{split}
& \hskip 5mm |n^{1/2}A_{nr} \{V^{(1)}(\theta^*)\}^{-1/2} \bigl \{H^{(1)}(\theta^*)+n^{-1} \triangledown_1^2(\theta^*)\bigr\}(\widehat{\theta}_a-\theta_a^*)|\\
& \leq  n^{1/2} || A_{nr}||\lambda_{max}\{V^{(1)}(\theta^*)\}^{-1/2}||\bigl\{H^{(1)}(\theta^*)+n^{-1} \triangledown_1^2(\theta^*)\bigr\}(\widehat{\theta}_a-\theta_a^*)|| =o_p(1).
\end{split}
\end{align*}
It follows that the vector $n^{1/2}A_{n} \{V^{(1)}(\theta^*)\}^{-1/2} \bigl\{H^{(1)}(\theta^*)+n^{-1} \triangledown_1^2(\theta^*)\bigr\}(\widehat{\theta}_a-\theta_a^*)$ converges to zero in probability. This means that $
n^{1/2}A_{n} \{V^{(1)}(\theta^*)\}^{-1/2}H^{(1)}(\theta^*)(\widehat{\theta}_a-\theta_a^*)=n^{-1/2}A_{n} \{V^{(1)}(\theta^*)\}^{-1/2}\triangledown_1 \ell_I(\theta^*)+o_p(1).
$
Next let $Z_{l}=n^{-1/2}A_n\{V^{(1)}(\theta^*)\}^{-1/2}\triangledown_1 \ell_I(\theta^*,Y_{(l)}).$
By the argument in the proof of Theorem 2 in Fan \& Peng (2004),   $\sum_{l=1}^n E|| Z_{l}||^2 I\{Z_{l}\geq \epsilon\}=o(1),$ and $\lim_{n} \sum_{l=1}^n \Cov(Z_{l})=G.$ According to the Lindeberg--Feller central limit theorem, this means that
$n^{-1/2}A_n\{V^{(1)}(\theta^*)\}^{-1/2}\triangledown_1 \ell_I(\theta^*) \to N(0, G)$ in distribution, completing the proof.
\end{proof}

\begin{proof}[of Theorem~\ref{thm3}]
\noindent According to
Lemma A\ref{uniform2},
$
\max_{s\in S_{-}} 2\{\ell_I(\widehat{\theta}_{s}) -
\ell_I(\theta_{s}^*)\}= O_p\{s_n\log (p_n)\}.
$
For the true model $T$,
$
2\{\ell_I(\widehat{\theta}_{T}) -
\ell_I(\theta_{T}^*) \} =O_p(1).
$
Define $\lambda_{T|s}(Y)=\ell_I(\theta_T^*;Y)-\ell_I(\theta_s^*;Y).$ Based on Lemma A\ref{sna}, we have
$
\max_{s\in S_{-}} \lambda_{T|s}(Y) -
E_{\theta_{T}^*}\{\lambda_{T|s}(Y)\}
=O_p[ \{n s_n\log (p_n)\}^{1/2}].
$
Therefore, for an under-fitting model,
\begin{eqnarray*}
-2\{\ell_I(\widehat{\theta}_{s})-\ell_I(\widehat{\theta}_{T})\}
&\geq&-2[\max_{s\in S_{-}} \{\ell_I(\widehat{\theta}_{s}) - \ell_I(\theta_{s}^*)\}]+2\{\ell_I(\widehat{\theta}_{T})-\ell_I(\theta_{T}^*)\}\\
&& \hskip 10mm +2 [\lambda_{T|s}(Y) -
E_{\theta_{T}^*}\{\lambda_{T|s}(Y)\}]
  +2 E_{\theta_{T}^*} \{\lambda_{T|s}(Y)\} \\
&=&O_p\{s_n\log(p_n)\}+O_p[\{n s_n \log (p_n)\}^{1/2}]+2E_{\theta_{T}^*}\{\lambda_{T|s}(Y)\}.
\end{eqnarray*}
Furthermore,
$
\min_{s\in S_{-}}\text{pseu-BIC(s)}-\mbox{pseu-BIC(T)}
\geq  \min_{s\in
S_{-}}-2\{\ell_I(\widehat{\theta}_{s})-\ell_I(\widehat{\theta}_{T})\}+
\gamma_n(d_s^*-d_T^*).
$
Because $|\gamma_n (d_s^*- d_T^*)|=O\{s_n\log (p_n)\},$ and $s_n^4\log (p_n)=o(n)$ by assumption,
$\liminf_{n\rightarrow \infty}\min_{s\in \mathcal{S}_{-}} E_{\theta_{T}^*}\{\lambda_{T|s}(Y)\}/\{n s_n\log (p_n)\}^{1/2}=\infty,$
and thus $\hbox{pr}_{\theta_{T}^*}\{\mbox{pseu-BIC}(T) <
\min_{s\in S_{-}}\mbox{pseu-BIC}(s)\}\rightarrow 1.$

For an over-fitting
marginal $s$,
\begin{eqnarray*}
\mbox{pseu-BIC}(s) - \mbox{pseu-BIC}(T)
 &=& -2 \{\ell_I(\widehat{\theta}_s) -
\ell_I(\widehat{\theta}_{T})\}
+ (d_s^* - d_T^*) \gamma_n   \\
&\geq& -\max_{s\in \mathcal{S}+}Q_{s/T}+(d_s^* - d_T^*) \gamma_n+o_p(1).
\end{eqnarray*}
By Lemma A\ref{uniform3}, $\hbox{pr}_{\theta_{T}^*}\{\max_{s\in \mathcal{S}+}Q_{s/T}<(d_s^* - d_T^*) \gamma_n \}\rightarrow 1$, completing the proof.
\end{proof}

\section*{Appendix 3}
\subsection*{Proofs of the technical lemmas}
\begin{proof}[of Lemma~\ref{expmomentmult}]
By Taylor expansion, for $|| t||\leq \delta,$ the cumulant generating
function for $Z_i$ is
$
g_i(t)= t\trans t/2+1/6\sum_{j,k,l=1}^m \partial^3 g_i(t^*)/(\partial t_j \partial t_k \partial t_l) \,t_j t_k t_l, 
$
for some $0 \leq || t^*|| \leq || t|| \leq \delta.$
  Let $\partial^3 \overline{g}(t)/(\partial t_j \partial t_k \partial t_l)=n^{-1}\sum_{i=1}^n \partial^3 g_i(t)/(\partial t_j \partial t_k \partial t_l).$ Because each partial third order derivative is uniformly bounded, so too is the average partial third order derivative. For any
$|| t||/n^{1/2}\leq \delta,$ the moment generating function of
$\eta=n^{-1/2}\sum_{i=1}^n Z_i$ is
$$
\phi_{\eta}(t)=\exp\{t\trans t/2+ 1/6\sum_{i,j,k=1}^m 
n^{-1/2}\partial^3 \overline{g}(t^*/n^{1/2})/(\partial t_i \partial t_j \partial t_k) t_i t_j t_k\}=\exp[(t\trans t/2)\{1+o(1)\}].
$$ This is due to the fact that $|| t^*||<|| t||<\{s_n^2\log (p_n)\}^{1/2}$, and $|\partial^3 \overline{g}(t^*/n^{1/2})/(\partial t_i \partial t_j \partial t_k)|<C$ as  $|| t||/n^{1/2}\leq n^{-1/2}\{s_n^2\log (p_n)\}^{1/2}\rightarrow 0.$ Therefore,
$\log[ E\{\exp(t\eta)\}]\leq a^2 t\trans t/2$ for $|| t||< \{s_n^2\log (p_n)\}^{1/2},$ for some $a^2>1,$ and $n$  sufficiently large.

\end{proof}

\begin{lemma}
\label{uniest} Under regularity Conditions 1--3 and Assumptions 1--4, there exists a solution $\widehat{\theta}_s$ to the score equation $U_n(\theta_s;Y)=\partial \ell_I(\theta_s,Y)/\partial \theta_s=0$ such that it falls within an $\{s_n^2 \log (p_n)/n\}^{1/2}$
neighborhood of $\theta_{s}^*$ for all $s\in \mathcal{S}$ with
probability tending to $1,$ as $n\rightarrow \infty.$
\end{lemma}

\begin{proof}[of Lemma~A\ref{uniest}]
For any unit vector $v,$ let
$\theta_s=\theta_{s}^*+C\{s_n^2 \log (p_n)/n\}^{1/2}v,$ for some constant $C.$ By Taylor expansion,
\begin{eqnarray*}
\ell_I(\theta_s)-\ell_I(\theta_{s}^*)
=C \{s_n^2 \log (p_n)/n\}^{1/2}
v\trans U_n(\theta_{s}^*)+(1/2)C^2 \{s_n^2 \log (p_n)/n\} v\trans \ell_I^{(2)}(\widetilde{\theta}_{s})v,
\end{eqnarray*}
where $\widetilde{\theta}_s$ is within the $\eta$-neighborhood
of $\theta_{s}^*,$ and $\ell_I^{(2)}=\partial^2 \ell_I(\theta_s)/\partial \theta_s^2.$ By the regularity condition that $E\{-\ell_I^{(2)}(\widetilde{\theta}_{s})\}$ has eigenvalues uniformly bounded away from zero and infinity, when $\theta_s$ is in the $\eta$-neighborhood of $\theta_s^*,$ we have
 $ v\trans E\{-\ell_I^{(2)}(\widetilde{\theta}_{s})\}v=O_p(n).$ Using
arguments similar to those in the proof of Lemma A~\ref{sna}, we have
$$
\max_{s \in \mathcal{S}}|\ell_{ij}^{(2)}(\widetilde{\theta}_{s})-E\{\ell_{ij}^{(2)}(\widetilde{\theta}_{s})\}|=O_p[n^{1/2}s_n^{1/2}\{\log
(p_n)\}^{1/2}]=o_p[E\{\ell_{ij}^{(2)}(\widetilde{\theta}_{s})\}],
$$ where $\ell_{ij}^{(2)}$ denotes the second derivative of $\ell_I$ respect to index $i$ and $j$ where $i,j \in \{(vw),v=1,\ldots,K, w=1,\ldots,d_s\}.$
From Lemma A\ref{sna}, we have
$\max_s || U_n(\theta_{s,0})||=\{n s_n^2 \log (p_n)\}^{1/2}.$ By the Cauchy--Schwarz inequality, we have $v\trans U_n(\theta_{s,0})\leq || v||*|| U_n(\theta_{s,0})||=O_p[\{n s_n^2\log (p_n)\}^{1/2}].$ Combining the results above, we have
$
\max_{s\in \mathcal{S}} \{\ell_I(\theta_s)-\ell_I(\theta^*_{s})\}
< 0
$
 in probability with constant $C$ chosen sufficiently large.
This means that
$
\pr[\max_{s\in
\mathcal{S}}\{\ell_I(\theta_s)-\ell_I(\theta_{s}^*)\}<0]\rightarrow
1,\,\,\text{as}\,n\rightarrow \infty.
$
Thus with
probability tending to $1,$  there exists solution to the
 score equation such that it falls within an $\{s_n^2 \log (p_n)/n\}^{1/2}$
neighborhood of $\theta_{s}^*$ for all $s\in \mathcal{S}.$
\end{proof}

\begin{lemma}
\label{uniformnew}
Under regularity Conditions 1--3 and Assumptions 1--4, $
2\{\ell_I(\widehat{\theta}_{s})-\ell_I(\theta_{s}^*)\}
=Q_s\{1+o_p(1)\},$ where $Q_s=n^{-1} U_n(\theta_{s}^*)\trans
H_{s}(\theta_{s}^*)^{-1}U_n(\theta_{s}^*),$ and $o_p(1)$ holds for all models
$s\in \mathcal{S}.$
\end{lemma}
\begin{proof}[of Lemma~A\ref{uniformnew}]
Consider a competing model $s.$ Let $\ell^{(1)}_{r}$ denote $\partial \ell_I/\partial \theta_{r},$ $\ell^{(2)}_{rt}$ denote $\partial^2 \ell_I/\partial\theta_{r}\partial\theta_{t},$ $\ell^{(3)}_{rtu}$ denote $\partial^3 \ell_I/\partial\theta_{r}\partial\theta_{t}\partial\theta_{u},$ for $r,t, u \in \{(vw), v=1,\ldots, K, w=1,\ldots,d_s\}.$ Let $H_{rt}(\theta_s^*)$ denote the $(r,t)$th entry of the Hessian matrix.
A Taylor expansion of
$\ell^{(1)}_r(\widehat{\theta}_s)=0$ around
$\theta_{s}^*$ gives the system of equations
\begin{align*}
\begin{split}
0=n^{-1} \ell^{(1)}_r(\widehat{\theta}_s
)=&n^{-1} \ell^{(1)}_r(\theta_{s}^*)+\sum_t n^{-1} \ell^{(2)}_{rt}(\theta_{s}^*)(\widehat{\theta}_s-\theta_{s}^*)_{[t]}\\
&+\sum_{tu} (2n)^{-1}
\ell^{(3)}_{rtu}(\widetilde{\theta}_{s})(\widehat{\theta}_s-\theta_{s}^*)_{[t]}(\widehat{\theta}_s-\theta_{s}^*)_{[u]},
\end{split}
\end{align*}
 and for some $\widetilde{\theta}_s$ between
$\theta_{s}^*$ and $\widehat{\theta}_s.$

Here $ n^{-1} \sum_t\ell^{(2)}_{rt}(\widehat{\theta}_s-\theta_{s}^*)_{[t]}=
\sum_t \{ -H_{rt}+(n^{-1} \ell^{(2)}_{rt}+H_{rt})\}(\widehat{\theta}_s-\theta_{s}^*)_{[t]},$ where $\ell^{(2)}_{rt}$ and $H_{rt}$ are evaluated at $\theta_{s}^*.$ By Lemma A\ref{sna}, 
$\max_s (n^{-1} \ell^{(2)}_{rt}+H_{rt})=\{s_n\log (p_n)/n\}^{1/2}=o_p(1),$ 
we can rewrite $\sum_t n^{-1} \ell^{(2)}_{rt}(\widehat{\theta}_s-\theta_{s}^*)_{[t]}=
\sum_t ( -H_{rt})(\widehat{\theta}_s-\theta_{s}^*)_{[t]}\{1+o_p(1)\}.$ By a similar argument, we have $n^{-1}
\ell^{(3)}_{rtu}(\widetilde{\theta}_{s})=E\{n^{-1} \ell^{(3)}_{rtu}(\widetilde{\theta_{s}})\}\{1+o_p(1)\}.$ We rewrite
$\sum_{tu} (2n)^{-1}
\ell^{(3)}_{rtu}(\widetilde{\theta}_{s})(\widehat{\theta}_s-\theta_{s}^*)_{[t]}(\widehat{\theta}_s-\theta_{s}^*)_{[u]}
=\sum_{t}[ (1/2)
(\widehat{\theta}_s-\theta_{s}^*)_{[t]}\sum_u \{n^{-1} \ell^{(3)}_{rtu} (\widetilde{\theta}_{s})(\widehat{\theta}_s-\theta_{s}^*)_{[u]}\}].$ By Lemma A\ref{uniest},
$\sum_u n^{-1} \ell^{(3)}_{rtu}(\widetilde{\theta}_{s}) (\widehat{\theta}_s-\theta_{s}^*)_{[u]}=O_p[\{s_n^4 \log (p_n)/n\}^{1/2}]=o_p(1).$ This means that
\begin{align*}
0=n^{-1} \ell^{(1)}_r(\widehat{\theta}_s)=&n^{-1} \ell^{(1)}_r(\theta_{s}^*)-\sum_t H_{rt}(\widehat{\theta}_s-\theta_{s}^*)_{[t]}\{1+o_p(1)\}.
\end{align*}
We thus obtain $n^{-1} U_n(\theta_{s}^*)=H_s(\theta_s^*) (\widehat{\theta_{s}}-\theta_{s}^*)\{1+o_p(1)\},$ and $(\widehat{\theta_{s}}-\theta_{s}^*)=n^{-1} H_s^{-1}(\theta_s^*) U_n(\theta_{s}^*)\{1+o_p(1)\},$ while $o_p(1)$ holds for all models $s$. Next, Taylor expansion for the pseudo-loglikelihood leads to
$
\ell_I(\widehat{\theta}_s)-\ell_I(\theta_{s}^*)
=U_n(\theta_{s}^*)\trans(\widehat{\theta}_s-\theta_{s}^*)- (1/2)
\sum_{rt}n (\widehat{\theta}_s-\theta_{s}^*)_{[r]}(\widehat{\theta}_s-\theta_{s}^*)_{[t]}H_{rt}+\widetilde{R}_n,
$
where the error term is given by
\begin{eqnarray*}
\widetilde{R}_n&=& (1/2)
\sum_{rt}(\widehat{\theta}_s-\theta_{s}^*)_{[r]}(\widehat{\theta}_s-\theta_{s})_{[t]}(\ell^{(2)}_{rt}+n H_{rt})\\
&& \hskip 10mm +  (1/6)
\sum_{rtu}(\widehat{\theta}_s-\theta_{s}^*)_{[r]}(\widehat{\theta}_s-\theta_{s}^*)_{[t]}(\widehat{\theta}_s-\theta_{s}^*)_{[u]}
\ell^{(3)}_{rtu}(\widetilde{\theta}_s),
\end{eqnarray*}
where $\widetilde{\theta}_s$ is between $\theta_s^*$ and $\widehat{\theta}_s.$ By similar arguments as above, $(\ell^{(2)}_{rt}+n H_{rt})/(nH_{rt})=o_p(1)$ and $
\{\sum_{u}(\widehat{\theta}_s-\theta_{s}^*)_{[u]}
\ell^{(3)}_{rtu}(\widetilde{\theta}_s)\}/nH_{rt}=o_p(1).$
Therefore,
\begin{eqnarray*}
\ell_I(\widehat{\theta}_s)-\ell_I(\theta_{s}^*)
=U_n(\theta_{s}^*)\trans(\widehat{\theta}_s-\theta_{s}^*)- \{(1/2)
\sum_{rt}n (\widehat{\theta}_s-\theta_{s}^*)_{[r]}(\widehat{\theta}_s-\theta_{s}^*)_{[t]}H_{rt}\}\{1+o_p(1)\}.
\end{eqnarray*}
This implies that $
2\{\ell_I(\widehat{\theta}_{s})-\ell_I(\theta_{s}^*)\}
=n^{-1} U_n(\theta_{s}^*)\trans
H_s(\theta_{s}^*)^{-1}U_n(\theta_{s}^*)\{1+o_p(1)\},$ where $o_p(1)$ holds for all the models
$s\in \mathcal{S}.$
\end{proof}

\begin{lemma}
\label{uniform2} Under regularity Conditions 1--3 and Assumptions 1--4,
 $$\max_{s \in \mathcal{S}}|2\{\ell_I(\widehat{\theta}_{s})-\ell_I(\theta_{s}^*)\}|
 =O_p\{s_n \log (p_n)\}.$$
\end{lemma}

\begin{proof}[of Lemma~A\ref{uniform2}]
By Lemma A\ref{uniformnew}, we have $2 \{\ell_I(\widehat{\theta}_s)-\ell_I(\theta_{s}^*)\}=Q_s\{1+o_p(1)\},$ where the quadratic approximation $Q_s=n^{-1} U_n(\theta_{s}^*)\trans H_s(\theta_{s}^*)^{-1}U_n(\theta_{s}^*)$ and the term $o_p(1)$ holds true uniformly for every model $s.$ Therefore, it suffices to show that $\max_{s\in \mathcal{S}}|Q_s|=O_p\{s_n \log (p_n)\}.$
Based on the cumulant boundedness condition of $\ell^{(1)}(\theta_s^*;Y_{(i)})$ and the uniform boundedness of the eigenvalues of $V_s(\theta_s)$ in Assumption 3, we have $\eta= n^{-1/2}\{V_s(\theta_{s}^*)\}^{-1/2} U_n(\theta_{s}^*)$ satisfying the exponential moment condition
$$
\log[ E\{ \exp(\gamma\trans \eta)\}]\leq a^2 || \gamma||^2/2,
$$
with $\gamma \in R^{d_s},$ $||\gamma||\leq \{s_n^2\log (p_n)\}^{1/2},$ and some constant $a^2>1.$ We scale the vector $\eta,$ as $\eta^*=\eta/a,$ so that
$
\log [E\{ \exp(\gamma\trans \eta^*)\}] \leq ||\gamma||^2/2,
$
with $||\gamma||\leq \{a^2 s_n^2 \log (p_n)\}^{1/2}=g.$ Define $B=V_s^{1/2}(\theta_{s}^*)H_s(\theta_{s}^*)^{-1} V_s^{1/2}(\theta_{s}^*)$ and $\tau=\lambda_{\max}(B).$ Because the eigenvalues of $H_s(\theta_{s}^*)$ and $V_s(\theta_{s}^*)$ are uniformly bounded away from $0$ and infinity, $\tau$ is bounded by a constant. We scale the matrix and let $B^*=B/\tau.$ Then the maximum eigenvalue of $B^*$ is $1.$ After the scaling, we have $Q_s=a^2\tau (\eta^*)\trans B^* \eta^*=a^2\tau Q_s^*,$ where $Q_s^*=(\eta^*)\trans B^* (\eta^*).$

Next we apply the large deviation result from Corollary 4.2 of Spokoiny \& Zhilova (2013).
Let $p_G=\text{tr}(B^*)$ and $v_G^2=2 \text{tr}\{(B^*)^2\}.$ Because $g^2=a^2 s_n^2 \log( p_n),$  we have $g^2> 2p_G.$ Define $w_c$ by
$
 w_c(1+w_c)/(1+w_c^2)^{1/2}
=gp_G^{-1/2}.
$
Define $\mu_c=\min\{w_c^2/(1+w_c^2), (2/3)\}.$ Further define $y_c^2=(1+w_c^2)p_G,$ and $2x_c=\mu_c y_c^2+\log [\text{det} \{I_{d_s}-\mu_c (B^*)^2\}].$
Because $p_G=O(d_s)$, and $v_G^2=O(d_s),$ and the eigenvalues of $B^*$ are all bounded away from zero uniformly, according to Spokoiny \& Zhilova (2013), $x_c > g^2/4,$ for $n$ sufficiently large. For $v_G/18\leq x\leq x_c,$
$
\pr\{Q_s^*\geq (p_G+6x)\}\leq 2 e^{-x}+ 8.4 e^{-x_c}.
$
Choosing $x=(7/6) s_n \log (p_n),$ we have $x<x_c.$ Then
$$
\pr[Q_s^*\geq \{p_G+7 s_n \log (p_n)\}]\leq 10.4 \exp\{-(7/6) s_n \log (p_n)\}.
$$
By the Bonferroni inequality, 
\begin{align*}
\begin{split}
&\max_{s\in \mathcal{S}}\pr\{|Q_s^*|> 8 s_n \log (p_n)\}\leq \sum_s \pr\{|Q_s^*|>p_G+7 s_n \log (p_n)\}\\
&=10.4 \exp\{-(7/6) s_n \log(p_n)\} p_n^{s_n}\rightarrow 0.
\end{split}
\end{align*}
 This means that $Q_s$ is $O_p\{s_n\log (p_n)\}$ uniformly for all $s.$
\end{proof}

\begin{lemma}
\label{uniform3} Under regularity Conditions 1$\verb+--+$3 and Assumptions 1$\verb+--+$4, if $\gamma_n=6\omega(1+\gamma)\log (p_n)$ for some $\gamma>0$ or $\gamma_n=6\omega\{\log (p_n) +\log \log (p_n)\}$, then
 ${\rm pr}\{\max_{s \in S_+}Q_{s/T}/(d_s^*-d_T^*)\geq
 \gamma_n\}=o(1).$
\end{lemma}

\begin{proof}[of Lemma~A\ref{uniform3}]

Let $\eta_s=V_s(\theta_{s}^*)^{-1/2} U_n(\theta_{s}^*).$ Based on Assumption 4 and Lemma \ref{expmomentmult}, we have
$
\log [E \{\exp(\gamma\trans \eta_s)\}]\leq a^2||\gamma||^2/2,
$
with $\gamma \in R^{d_s},$ $||\gamma||^2\leq s_n^2 \log (p_n),$ and some constant $a^2>1.$ We scale the vector $\eta_s,$ and let $\eta^*_s=\eta_s/a,$ then we have
$
\log[ E\{ \exp(\gamma\trans \eta^*_s)\}]\leq ||\gamma||^2/2,
$
with $||\gamma||\leq \{a^2 s_n^2 \log (p_n)\}^{1/2}=g.$
Given the matrix $B_{s}=V_s^{1/2}(\theta_{s}^*)M_{s/T} V_s^{1/2}(\theta_{s}^*),$ $\text{tr}(B_s)=d_s^*-d_T^*.$ let $B_s^*=B_s/\tau,$ where $\tau=\lambda_{\max}(B_s).$ Then the maximum eigenvalue of $B_s^*$ is $1.$ After the scaling, we have $Q_{s/T}=a^2\tau Q_{s/T}^*,$ where $Q_{s_T}^*=(\eta_s^*)\trans B_s^* \eta_s^*.$ Define $p_G=\text{tr}(B_s^*),$ and $v_G=[2 \text{tr}\{(B_s^*)^2\}]^{1/2}.$ Using the inequality for the trace of matrix product (Fang et al., 1994), $v_G\leq (2 p_G)^{1/2}.$ Now we apply the large deviation result from Corollary 4.2 of Spokoiny \& Zhilova (2013) and obtain
$$
\pr\{Q^*_{s/T}>(p_G+K)\} \leq 2\exp(-K/6)+8.4\exp(-x_c), \,\,\text{if}\,\, 6x_c>K> v_G/3,
$$
where $x_c> g^2/4$  for large $n.$ Choosing $L=\{ (d_s^*-d_T^*)/\tau\} \{\gamma_n / (a^2-1)\},$  we have
$
\lim_{n \rightarrow \infty} L/(v_G/3)>1.
$
Furthermore, as $\gamma_n(d_s^*-d_T^*)=O\{s_n \log (p_n)\},$ then $L\leq 6x_c.$ Using the relationship $d_s^*-d_T^*=(d_s-d_T)\overline{\tau},$ $p_G=(d_s^*-d_T^*)/\tau,$ and the Bonferroni inequality, 
\begin{align*}
\begin{split}
&\pr\{\max_{s\in S^+} Q_{s/T}>(d_s^*-d_T^*) \gamma_n\}\\
&\leq \sum_{s\in S+}\pr\{Q_{s/T}^*>(d_s^*-d_T^*) \gamma_n/(a^2\tau)\} \\
&=\sum_{s\in S+}\pr\{Q_{s/T}^*>p_G+p_G(\gamma_n / a^2-1)\} \\
&\leq\sum_{d_s=d_T+1}^{p_n} C(p_n-d_T,d_s-d_T)10.4  \exp\{-(\gamma_n/a^2-1)(d_s-d_T)\overline{\tau}/(6\tau)\} \\
&\leq \sum_{m'=1}^{p_n-d_T}C(p_n-d_T,m')10.4\exp\{-m' (\gamma_n/a^2-1)/(6w)\},\quad \text{with}\, m'=d_s-d_T, \\
&\leq [1+10.4\exp\{-(\gamma_n/a^2-1)/(6w)\}]^{p_n-d_T}-1. 
\end{split}
\end{align*}
Because $a^2$ can be chosen as close to $1$ as possible with increasing sample size $n$, it can be seen that the choices of
 $\gamma_n=6w(1+\gamma)\log (p_n)$ or $\gamma_n=6w\{\log (p_n)+\log \log (p_n)\}$, lead to
 $\lim_{n\rightarrow \infty} [1+10.4\exp\{- (\gamma_n /a^2-1)/(6w)\}]^{p_n-d_T}=1.$
\end{proof}

\begin{lemma}
\label{basic1} Let $Z_i$, $i=1,2,\ldots,n,$ be independent random variables. If each $Z_i$ has zero mean, unit variance and satisfies the cumulant boundedness condition in Definition 1, then
$$
{\rm pr}\left[\sum_{i=1}^n Z_n>\{2n s_n \log (p_n)\}^{1/2}\right]=o\{p_n^{-s_n}\}.
$$
\end{lemma}
\begin{proof}[of Lemma~A\ref{basic1}]
By Taylor expansion, for $|t|\leq \delta,$ the cumulant generating
function for $Z_i$ is
$$
g_i(t)=t^2/2+g_i^{(3)}(t^*)t^3/6,
$$
for some $0 \leq |t^*| \leq |t| \leq \delta.$ Let $\overline{g}^{(3)}(t)=\sum_i g_i^{(3)}(t)/n.$ Because each $g_i^{(3)}$ is uniformly bounded,
the average $\overline{g}^{(3)}$ is also bounded. For any
$|t|/n^{1/2}\leq \delta,$ the moment generating function of
$n^{-1/2}\sum_{i=1}^n Z_i$ is equal to
$$
\phi_n(t)=\exp\{t^2/2+ \overline{g}^{(3)}(t^*/n^{1/2}) t^3/(6 n^{1/2})\}.
$$
For convenience, let $b_n=\{2.1 s_n\log (p_n)\}^{1/2}.$ It can be shown that
$$
\hbox{I}\{(n^{-1/2}\sum_{i=1}^n Z_i)> b_n\}\leq \exp\{t(n^{-1/2}\sum_{i=1}^n
Z_i- b_n)\},
$$
for any $t>0.$ Then
\begin{align*}
\begin{split}
\pr(n^{-1/2}\sum_{i=1}^n Z_i>b_n)&\leq
E[\exp\{t(n^{-1/2}\sum_{i=1}^n Z_i- b_n)\}]\\
&=\exp \{t^2/2+\overline{g}^{(3)}(t^*/n^{1/2}) t^3/(6n^{1/2})-b_n t\}
=\exp[(t^2/2)\{1+o(1)\}-b_n t].
\end{split}
\end{align*}
Letting $t=b_n,$ 
$$
\pr[\sum_{i=1}^n Z_i>\{2.1 ns_n\log (p_n)\}^{1/2}]\leq \exp[-(1/2)b_n^2\{1+o(1)\}]=o(p_n^{-s_n}).
$$

\end{proof}

\begin{lemma}
\label{sna}  Under regularity Conditions 1--3 and Assumptions 1--4, 
\begin{align*}
\begin{split}
&(1)\,\, \max_{s\in
\mathcal{S}}|\sum_{i=1}^n \ell_I(\theta_s^*;Y_{(i)})-E\{\ell_I(\theta_s^*;Y_{(i)})\}|=O_p[n^{1/2}s_n^{1/2}\{\log (p_n)\}^{1/2}];\\
&(2)\,\, \max_{s\in
\mathcal{S}}|\sum_{i=1}^n \partial \ell_I(\theta_s^*;Y_{(i)})/\partial \theta_j|=O_p[n^{1/2}s_n^{1/2}\{\log (p_n)\}^{1/2}];\\
&(3)\,\, \max_{s\in
\mathcal{S}}|\sum_{i=1}^n \partial^2 \ell_I(\theta_s^*;Y_{(i)})/\partial \theta_j \partial \theta_k-E\{\partial^2 \ell_I(\theta_s^*;Y_{(i)})/\partial \theta_j \partial \theta_k\}|\\
&\,\,\,\,\,\,\,\,=O_p[n^{1/2}s_n^{1/2}\{\log (p_n)\}^{1/2}];\\
&(4)\,\, \max_{s\in \mathcal{S}}|\sum_{i=1}^n
\partial^3 \ell_I(\theta_s;Y_{(i)})/\partial \theta_j \partial \theta_k \partial \theta_l-E\{\partial^3 \ell_I(\theta_s;Y_{(i)})/\partial \theta_j \partial \theta_k \partial \theta_l\}|\\&\,\,\,\,\,\,\,\,=O_p[n^{1/2}s_n^{1/2}\{\log (p_n)\}^{1/2}],
\end{split}
\end{align*}
with $j,k,l \in \{vw, v=1,\ldots,K,w=1,\ldots,d_s\},$  $||\theta_s-\theta_s^*||\leq \delta.$
\end{lemma}

\begin{proof}[of Lemma~A\ref{sna}]
Because $\ell_I(\theta_s^*;Y_{(i)})$ satisfies the cumulant boundedness condition in Definition 1, its first and second moments are bounded uniformly.
Given a model $s,$ by Lemma A\ref{basic1}, $$
\pr\big(\sum_{i=1}^n [\ell_I(\theta_s^*;Y_{(i)})-E\{\ell_I(\theta_s^*;Y_{(i)})\}]/\Var\{\ell_I(\theta_s^*;Y_{(i)})\}>\{2.1 ns_n\log (p_n)\}^{1/2}\big)=o(p_n^{-s_n}).$$ Because there are $p_n^{s_n}$ models in the model space, by the Bonferroni inequality, $$
\pr\big( \max_{s\in \mathcal{S}} \sum_{i=1}^n [\ell_I(\theta_s^*;Y_{(i)})-E\{\ell_I(\theta_s^*;Y_{(i)})\}]>C\{2.1 n s_n\log (p_n)\}^{1/2}\big)\leq o(p_n^{-s_n}) p_n ^{s_n} \rightarrow 0,$$ where $C$ is the upper bound for $\Var\{\ell_I(\theta_s^*;Y_{(i)})\}.$ Similar arguments apply to the result for the first, second and third derivatives of the pseudo-loglikelihood.
\end{proof}

\begin{center}
R{\scriptsize REFERENCES}
\end{center}

\begin{description}

\item
\hspace{-30pt}

\item
\hspace{-30pt}
B{\scriptsize ACH}, F. R. (2008). Consistency of the group lasso and multiple kernel learning.  \emph{J. Mach. Learn. Res.} \textbf{9}, 1179-225.

\item \hspace{-30pt} B{\scriptsize REHENY}, P. \& H{\scriptsize
UANG}, J. (2015). Group descent algorithms for nonconvex penalized linear and logistic regression models with grouped predictors. \emph{Stat. Comput.} \textbf{25}, 173-87.

\item \hspace{-30pt} C{\scriptsize HEN}, J. H. \& C{\scriptsize
HEN}, Z. H. (2008). Extended Bayesian information criteria for
model selection with large model spaces. \emph{Biometrika}
\textbf{95}, 759-71.

\item \hspace{-30pt} C{\scriptsize OX}, D. R. \& R{\scriptsize
EID},  N. (2004). A note on pseudolikelihood constructed from
marginal densities. \emph{Biometrika} \textbf{91}, 729-37.

\item \hspace{-30pt} F{\scriptsize AN}, J. \& L{\scriptsize I}, R. (2001). Variable selection via nonconcave penalized likelihood and it oracle properties. \emph{J. Am. Statist. Assoc.} \textbf{96}, 1348-60.

\item \hspace{-30pt} F{\scriptsize AN}, J. \&
P{\scriptsize ENG}, H. (2004). Nonconcave penalized likelihood with a diverging number of parameters. \emph{Ann. Statist.} \textbf{32}, 928-61.

\item \hspace{-30pt} F{\scriptsize AN}, J. \&
L{\scriptsize V}, J. (2010). A selective overview of variable selection in high dimensional feature space. \emph{Stat. Sinica.} \textbf{20}, 101-148.

\item \hspace{-30pt} F{\scriptsize ANG}, Y., L{\scriptsize OPARO}, K. A. \&
F{\scriptsize ENG}, X. (1994) Inequalities for the trace of Matrix Product.
\emph{IEEE T. Automat. Contr.} \textbf{39}, 2489-2490.

\item \hspace{-30pt} F{\scriptsize OSTER}, D. P., \& G{\scriptsize EORGE}, E. I. 
 (1994) The risk inflation criterion for multiple regression.
\emph{Ann. Statist.} \textbf{22}, 1947-1975.


\item \hspace{-30pt} G{\scriptsize AO}, X. \& S{\scriptsize
ONG}, P. X.-K. (2010). Composite likelihood Bayesian information criteria for model
  selection in high-dimensional data. \emph{J. Am. Statist. Assoc.}
\textbf{105}, 1531--40.


\item \hspace{-30pt} G{\scriptsize ODAMBE}, V. P. (1960).
An optimum property of regular maximum likelihood estimation. \emph{Ann. Math. Statist.} \textbf{31}, 1208-11.

\item
\hspace{-30pt}
G{\scriptsize UO}, X., Z{\scriptsize HANG}, H., W{\scriptsize ANG}, Y. \& W{\scriptsize U}, J.  (2015). Model selection and estimation in high dimensional regression models with group \SCAD.  \emph{Stat. Probabil. Lett.} \textbf{103}, 86-92.

\item
\hspace{-30pt}
H{\scriptsize UANG}, J., B{\scriptsize REHENY}, P. \& M{\scriptsize A}, S. (2012). A selective review of group selection in high-dimensional models.  \emph{Statist. Sci.} \textbf{27}, 481-99.

\item
\hspace{-30pt}
I{\scriptsize WAMOTO}, T., B{\scriptsize IANCHINI}, G., B{\scriptsize OOSER}, D. \& Q{\scriptsize I}, Y. et al. (2011). Gene pathways associated with prognosis and chemotherapy sensitivity in molecular subtypes of breast cancer  \emph{J Natl Cancer Inst} \textbf{103}, 264-72.

\item
\hspace{-30pt}
J{\scriptsize OE}, H. \& L{\scriptsize EE}, Y. (2009). On weighting of bivariate margins in pairwise likelihood. \emph{J. Multivariate Anal.} \textbf{100}, 670-85.

\item
\hspace{-30pt}
K{\scriptsize IM}, Y.,  K{\scriptsize
WON},  S. \&  C{\scriptsize HOI},  H. (2012).
Consistent Model selection criteria on high dimensions. \emph{J. Mach. Learn. Res.} \textbf{13}, 1037-57.

\item
\hspace{-30pt}
K{\scriptsize WON}, S.,  \& K{\scriptsize IM},  Y.  (2012).
Large sample properties of the scad-penalized maximum likelihood estimation on high dimensions. \emph{Stat. Sinica} \textbf{22}, 629-53.

\item
\hspace{-30pt}
L{\scriptsize INDSAY}, B. G. (1988). Composite likelihood methods.
\emph{Statistical inference from stochastic processes},
 Ed. Prabhu, N. U., 221-239. Providence, RI: American Mathematical
 Society.

\item
\hspace{-30pt}
L{\scriptsize INDSAY}, B. G., Y{\scriptsize I}, G. Y. \& S{\scriptsize UN}, J. (2011). Issues and strategies in the selection of composite likelihoods.
\emph{Stat. Sinica} \textbf{21}, 71-105.

\item
\hspace{-30pt}
M{\scriptsize EIER}, L., V{\scriptsize AN DE GEER}, S. \& B{\scriptsize \"{U}HLMANN}, P.  (2008). The
group Lasso for logistic regression. \emph{J. R. Stat. Soc. Ser. B} \textbf{70}, 53–71.

\item
\hspace{-30pt}
N{\scriptsize ARDI}, Y. \& R{\scriptsize INALDO}, A.  (2008).  On the asymptotic properties of the group lasso estimator for linear models.  \emph{Electron. J. Stat.} \textbf{2}, 605–33.

\item \hspace{-30pt} R{\scriptsize IBATET}, M.,  C{\scriptsize
OOLEY}, D. \& D{\scriptsize AVISON}, A. C. (2012). Bayesian Inference from composite likelihood, with an application to spatial extremes.
\emph{Stat. Sinica} \textbf{22}, 813-45.

\item \hspace{-30pt}
S{\scriptsize POKOINY}, V. \& Z{\scriptsize HILOVA}, M. (2013). Sharp deviation bounds for quadratic forms. \emph{Math. Method. Statist.} \textbf{22}, 100-13.

\item \hspace{-30pt} T{\scriptsize IBSHIRANI}, R. (1996). Regression shrinkage and selection via the lasso. \emph{J. Royal. Statist. Soc. B} \textbf{58}, 267-88.

\item \hspace{-30pt} V{\scriptsize ARIN}, C. \& V{\scriptsize
IDONI}, P. (2005). A note on composite likelihood inference and
model selection. \emph{Biometrika} \textbf{92}, 519-28.

\item \hspace{-30pt} V{\scriptsize ARIN}, C. \& V{\scriptsize IDONI}, P. (2006). Pairwise likelihood inference for ordinal categorical time series. \emph{Comput. Statist. Data Anal.} \textbf{51}, 2365-73.

\item
\hspace{-30pt}
V{\scriptsize ARIN}, C. (2008).
On composite marginal likelihoods.
\emph{Adv. Statist. Anal.} \textbf{92}, 1-28.

\item
\hspace{-30pt}
W{\scriptsize ANG}, L., L{\scriptsize I}, H. \& H{\scriptsize UANG}, J. Z. (2008). Variable selection in nonparametric varying-coefficient models for analysis of repeated measurements. \emph{J. Am. Statist. Assoc.} \textbf{103}, 1556-69.

\item
\hspace{-30pt}
W{\scriptsize ANG}, Y., K{\scriptsize LIJN}, J. G., Z{\scriptsize HANG}, Y. \& S{\scriptsize IEUWERTS}, A. M. et al. (2005). Gene-expression profiles to predict distant metastasis of lymph-node-negative primary breast cancer. \emph{Lancet} \textbf{365}, 671-9.

\item
\hspace{-30pt}
W{\scriptsize HITE}, H. (1982). Maximum likelihood estimation of
misspecified models. \emph{Econometrika} \textbf{50}, 1-25.

\item
\hspace{-30pt}
X{\scriptsize U}, X. \& R{\scriptsize EID}, N. (2011).
On the robustness of maximum composite likelihood estimate.
\emph{J. Statist. Plann. Inference} \textbf{141}, 3047-54.

\item \hspace{-30pt} Y{\scriptsize I}, G. Y. \& T{\scriptsize
HOMPSON}, M. E.  (2005).
Marginal and association regression models for longitudinal binary data with drop-outs: A likelihood Approach. \emph{Can. J. Stat.}
\textbf{33}, 3-20.

\item
\hspace{-30pt}
Y{\scriptsize UAN}, M. \& L{\scriptsize IN}, Y. (2006).   Model selection and estimation
in regression with grouped variables. \emph{J. Royal.
Statist. Soc. B} \textbf{68},   49–67.

\item \hspace{-30pt} Z{\scriptsize HANG}, C.-H. (2010). Nearly unbiased variable selection under minimax concave penalty. \emph{Ann. Statist.} \textbf{38}, 894–942.

\item \hspace{-30pt} Z{\scriptsize HANG} Y. \& S{\scriptsize HEN},
X. (2010). Model selection procedure for high-dimensional data. \emph{Stat. Anal. Data Min.} \textbf{3}, 350-58.

\item
\hspace{-30pt}
Z{\scriptsize HAO}, P., R{\scriptsize OCHA}, G \& Y{\scriptsize U}, B (2009). The composite absolute penalties family for grouped and hierachical variable selection.  \emph{Ann. Statist.} \textbf{37}, 3468-97.

\end{description}


\end{document}